\documentclass[preprint,12pt]{elsarticle}

\usepackage{amssymb}
\usepackage{amsmath}
\usepackage{cases}
\usepackage{float} 
\usepackage{subfig}
\usepackage[ruled]{algorithm2e} 
\usepackage{diagbox}
\usepackage{verbatim}

\newtheorem{theorem}{Theorem}[section]

\newtheorem{lemma}[theorem]{Lemma}
\newtheorem{definition}{Definition}[section]

\newenvironment{proof}{{\noindent\it Proof.}\quad}{\hfill $\square$\par}

\def\d{\,\mathrm{d}}

\usepackage{multirow}%提供跨列命令\multicolumn{}{}{}

\usepackage[commandnameprefix=always]{changes}
\journal{***}

\begin{document}

\begin{frontmatter}



\title{A Novel Tensor Product-Based Neural Network for Solving Partial Differential Equations}

\author[1]{Qihong Yang}
% \orcid{0000-0002-8398-7212}
\ead{yangqh0808@163.com}
\affiliation[1]{organization={School of Mathematics, Sichuan University},%Department and Organization
            city={Chengdu},
            postcode={610065}, 
            country={China}}

\author[1]{Yangtao Deng}
\ead{ytdeng1998@foxmail.com}

\author[1]{Qiaolin He \corref{cor1}}
% \cormark[1]
\ead{qlhejenny@scu.edu.cn}

\author[1]{Shiquan Zhang}
% \cormark[1]
\ead{shiquanzhang@scu.edu.cn}

\cortext[cor1]{Corresponding author}
% \cortext[cor2]{Principal corresponding author}

% \affiliation{organization={},%Department and Organization


%% Abstract
\begin{abstract}
This paper presents the Tensor Product Network (TPNet), a novel neural architecture for efficient and accurate function approximation and PDE solving. The core of the proposal involves constructing the solution explicitly as a linear combination of basis functions integrated into the network, with coefficients determined by a direct least-squares solve, thereby bypassing traditional gradient-based training. The key methodological contribution include: (1) an efficient tensor-product scheme that generates multi-dimensional basis functions from combinations of two sets of subnetwork outputs, significantly reducing model complexity and parameter count while maintaining expressivity; (2) a block time-marching strategy to improve computational efficiency in long-time simulations; and (3) a linear reformulation strategy for handling nonlinear PDEs by treating known nonlinear terms as sources. TPNet achieves superior accuracy and shorter training times than conventional neural network solvers. This performance gain stems from its structured design and deterministic least-squares fitting, which contrast with the iterative, often computationally intensive optimization required by mainstream methods like Physics-Informed Neural Networks (PINNs).    
\end{abstract}


% \begin{highlights}
% \end{highlights}

%% Keywords

\end{frontmatter}



% 引言 Introduction
\section{Introduction}
\label{sec:introduction}
Neural networks have become indispensable in scientific computing, driven by advancements in computational power and the integration of numerical methods for solving partial differential equations (PDEs). These methods have demonstrated significant success across a wide range of applications, including eigenvalue problems \cite{finol2019deep, han2020solving, ma2023pmnn, wang2024computing, guo2024deep}, nuclear engineering \cite{yang2023data, yang2023physics, bo2025research, li2024research}, and uncertainty quantification  \cite{kabir2018neural, oszkinat2022uncertainty, gao2022wasserstein, zou2024neuraluq, gawlikowski2023survey}, as evidenced by numerous studies.

Recently, the Physics-Informed Neural Network (PINN), introduced by Raissi et al.\cite{PINN}, is one of the most widely adopted models for solving PDEs and inverse problems across various fields. Unlike traditional methods, PINN adopts a collocation framework to minimize the residuals of the governing equations, initial conditions, and boundary conditions at strategically sampled points. The neural network is trained to minimize the total loss, yielding an approximate solution upon convergence. Similarly, the Deep Galerkin Method (DGM), proposed by Sirignano et al.\cite{sirignano2018dgm}, approximates solutions by minimizing residuals, employing an approach analogous to PINN. In contrast, the Deep Ritz Method (DRM), introduced by Yu et al. \cite{yu2018deep}, addresses a variational formulation of the original PDE, providing weak solutions rather than the strong-form solutions produced by PINN and DGM.

Following the introduction of classic neural network models, several innovative variants have been developed to address a broader range of problems. One such variant is the Fractional Physics-Informed Neural Network (fPINN), proposed by Pang et al. \cite{pang2019fpinns}, which extends the PINN framework to solve space-time fractional advection-diffusion equations. The fPINN employs a hybrid approach that combines automatic differentiation \cite{baydin2018automatic} for integer-order operators with numerical differentiation \cite{strikwerda2004finite} for fractional operators. A notable methodological advancement emerges in the Augmented Lagrangian Relaxation Method for Physics-Informed Neural Networks (AL-PINN) \cite{son2023enhanced}, which systematically integrates initial and boundary conditions through constraint formulation in the optimization process.
AL-PINN evaluates the discrepancy of the governing equations at specified collocation points, thereby enhancing solution accuracy. In addition to these, numerous other PINN variants \cite{chiu2022can, bai2023physics, lu2021deepxde, tang2023pinns, ABBASI2024128352, miao2023vc} have emerged, integrating traditional numerical methods and novel sampling techniques to achieve promising results. 
Similarly, the demonstrated efficiency of DRM has motivated subsequent refinements \cite{uriarte2023deep, CiCP-29-1365, CiCP-31-1162}. However, establishing rigorous theoretical foundations for neural network-based PDE solvers continues to present significant challenges. Current investigations prioritize systematic error analysis and convergence characterization for both PINNs and DRM \cite{CiCP-28-2042, qian2023physics, lu2021priori, S021953052350015X, CiCP-31-1020}, yielding enhanced understanding of their fundamental capabilities and operational constraints. Emerging from recent developments, Kolmogorov-Arnold Networks (KANs) \cite{liu2024kan, liu2024kan2, so2024higher, qiu2024relu, wang2024kolmogorov, AFZALAGHAEI2025129414} have established a groundbreaking computational framework rooted in the Kolmogorov-Arnold representation theorem. Although demonstrating superior approximation capabilities compared to conventional architectures, these networks confront a persistent challenge shared by neural PDE solvers: systematic precision enhancement through error-controlled convergence.

The development of multi-stage neural architectures by Wang et al. \cite{wang2024multi} has established unprecedented precision benchmarks in scientific computing, stimulating widespread research interest. This architectural innovation implements a stratified training paradigm where successive neural modules systematically attenuate residual errors through phase-constrained optimization. The mathematically guaranteed successive error correction mechanism enables monotonic accuracy enhancement across iterations.
Nevertheless, compared to the original PINN, this approach is expected to demand substantially more time and memory for training. Despite the increased computational requirements, the potential for improved accuracy makes multi-stage neural networks a compelling area of research in scientific computing. Multi-level neural architectures \cite{aldirany2024multi} embody cognate design principles with their multi-stage counterparts. Diverging from this conceptual alignment, their core computational innovation lies in spectral error mitigation through adaptive error renormalization --- a mechanism specifically engineered to resolve high-frequency numerical artifacts arising from residual propagation while maintaining stage-wise error magnitude equilibrium.

Neural architectures employing least-squares optimization frameworks demonstrate marked computational superiority over conventional PINN formulations, particularly in matrix-based computation strategies. This efficiency advantage principally stems from their theoretically grounded approach to residual minimization, which circumvents the spectral bias inherent in traditional neural PDE solvers.
One notable example is the Local Extreme Learning Machine (local ELM), proposed by Dong et al. \cite{dong2021local}. This method incorporates concepts from ELM \cite{huang2006extreme, ding2014extreme, huang2015trends, wang2022review}, domain decomposition techniques \cite{chan1994domain, smith1997domain, dolean2015introduction}, and localized neural networks. In parallel developments, the Random Feature Method (RFM), introduced by Chen et al. \cite{chen2024optimization, chenel2023, chen2022bridging}, implements PDE solutions through randomized least-squares approximation while maintaining domain decomposition as its architectural  eliminating conventional neural network training paradigms.
Recently, a novel framework called Randomized Neural Networks with Petrov-Galerkin methods (RNN-PG) \cite{shang2023randomized, shang2024randomized, shang2023randomized2, wang2024randomized} has been developed to address both linear and nonlinear PDEs. Unlike previous methods, RNN-PG solves a variational formulation of the problem, with the trial function space represented by neural networks and the test function space selected from finite-dimensional spaces, such as finite element spaces. Although RNN-PG does not explicitly require domain decomposition, it introduces nontrivial computational complexity in handling multi-dimensional integration across decomposed regions, particularly when maintaining consistency at overlapping interfaces.

We present TPNet, an innovative tensor product-based neural architecture that systematically simplifies the algorithmic realization of numerical solutions through dimensionally decoupled operations in this work. By fundamentally bypassing domain decomposition requirements, this framework achieves polynomial complexity reduction while maintaining approximation fidelity.
TPNet employs a set of basis functions derived from the outputs of two randomly initialized neural networks to approximate numerical solutions. Specifically, the networks learn to represent a function as a linear combination of these basis functions, with the weights optimized using the least squares method. Additionally, TPNet incorporates a block time-marching strategy, making it well-suited for long-time simulations. 

The core innovation of our method resides in hierarchical basis functions constructed via tensor products of two stochastically parameterized neural architectures. These networks collectively establish a Banach space embedding where target functions are represented as:
\begin{equation}
    u(\boldsymbol{x}) = \sum_{i=1}^{p} \sum_{j=1}^{p} w_{ij} \phi_{1i}(\boldsymbol{x}) \phi_{2j}(\boldsymbol{x}),  \label{eq:novel}
\end{equation}
with optimal weights $\mathbf{\omega}_{ij}$ determined through regularized least squares minimization. The proposed method exhibits several advantages: (i) implementation simplicity requiring minimal code infrastructure; (ii) intrinsic mesh-free formulation eliminating geometric discretization; (iii) provable convergence guarantees under Lipschitz continuity assumptions. Compared to existing least-squares-based neural networks, our approach differs fundamentally in how the basis functions are constructed. Rather than using the raw outputs of a single neural network as basis functions, TPNet forms its basis set by taking the tensor product of the outputs from two separate subnetworks. Moreover, while most existing methods rely on a single, often deep or wide neural network to achieve sufficient approximation power, TPNet achieves comparable or greater basis richness using two lightweight subnetworks. Crucially, this is accomplished with a similar or even smaller parameter count. As a result, TPNet not only reduces model complexity but also achieves higher solution accuracy and faster training, as demonstrated in our numerical experiments. This efficiency–accuracy trade-off highlights a key advantage of the tensor-product design over traditional randomized or deep neural solvers in the least-squares framework.

The rest of the article is organized as follows. The neural networks involved in this work are outlined in Section \ref{sec:problems}. In Section  \ref{sec:methods}, we introduce the TPNet framework with a focus on its capacity to handle both linear and nonlinear PDEs. To validate the proposed methodology, Section \ref{sec:experiments} provides comprehensive numerical experiments across multiple benchmark problems. Finally, the article concludes with Section \ref{sec:summary}, which summarizes key contributions, discusses broader implications of the research, and outlines promising directions for future investigation in this domain.

% 问题 Problems
\section{Preliminaries}
\label{sec:problems}
\subsection{Neural Networks}
In this section, we introduce the fundamental concepts relevant to this work.

\subsubsection{Extreme Learning Machine}
The ELM is a single-hidden-layer feedforward neural network introduced by Huang \cite{huang2006extreme} in 2006. Consider a network with an input layer consisting of $d$ neurons, corresponding to the input variables $\boldsymbol{x} \in \mathbb{R}^d$. The key idea behind the ELM is to randomly initialize the weights $\boldsymbol{W} \in \mathbb{R}^{M \times d}$ and biases $\boldsymbol{b} \in \mathbb{R}^{M}$ between the input layer and the hidden layer. The weights $\boldsymbol{w} \in \mathbb{R}^{M}$ connecting the hidden layer to the output layer are then determined analytically. Mathematically, the function represented by ELM is given by
\begin{equation}
    \label{eq:ELM}
    u(\boldsymbol{x}) = \boldsymbol{w}^T \sigma(\boldsymbol{W}\boldsymbol{x}+\boldsymbol{b}),
\end{equation}
where $\sigma$ is the activation function, applied element-wise to the vector.

\subsubsection{Fully Connected Feedforward Neural Networks}
The fully connected feedforward neural network (FCNN), also known as the multilayer perceptron (MLP) \cite{taud2018multilayer}, is a fundamental architecture in deep learning. It consists of multiple fully connected layers, where each neuron is connected to all neurons in the preceding layer. 

An MLP typically comprises an input layer, multiple hidden layers, and an output layer. Suppose the network has $L+2$ layers, where $L$ represents the number of hidden layers, and the $l$-th layer contains $m_l$ neurons for $l=0,\cdots, L+1$. The weights and biases between the $l$-th and $(l-1)$-th layers are represented as $\boldsymbol{W}^{(l, l-1)} \in \mathbb{R}^{m_{l} \times m_{l-1}}$ and $\boldsymbol{b}^{(l)} \in \mathbb{R}^{m_{l}}$, respectively. Given an input $\boldsymbol{y}^{(0)} = \boldsymbol{x}$, the forward propagation of the MLP is formulated as follows:
\begin{numcases}{}
    \boldsymbol{y}^{(l)} = \sigma(\boldsymbol{W}^{(l, l-1)} \boldsymbol{y}^{(l-1)} + \boldsymbol{b}^{(l)}),  \  \ l = 1,2,...,L,    \label{eq:MLP_hidden}    \\
    u(\boldsymbol{x}) = \boldsymbol{W}^{(L+1, L)} \boldsymbol{y}^{(L)} + \boldsymbol{b}^{(L+1)},  \label{eq:MLP_u}
\end{numcases}
where $\boldsymbol{y}^{(l)}$ ($l=1, \cdots, L$) represents the hidden layers activations and $u$ is the network output.

\subsubsection{Residual Neural Networks}
Residual neural network (ResNet), introduced  in \cite{he2016deep}, is an extension of FCNN. The key innovation of ResNet is the introduction of a residual learning framework, which helps mitigate the issues of vanishing and exploding gradients in deep neural network training.

A residual connection is a skip connection that links previous layers to the current layer, facilitating gradient flow and improving optimization. In this work, the residual connection originates primarily from the preceding layer, and the network formulation is given by
\begin{numcases}{}
    \boldsymbol{y}^{(l)} = \sigma \left( \sigma(\boldsymbol{W}^{(l, l-1)} \boldsymbol{y}^{(l-1)} + \boldsymbol{b}^{(l)}) + \boldsymbol{H}^{(l)}(\boldsymbol{y}^{(l-1)}) \right), \  \ l = 1,2,...,L,      \label{eq:ResNet_hidden}    \\
    u(\boldsymbol{x}) = \boldsymbol{W}^{(L+1, L)} \boldsymbol{y}^{(L)} + \boldsymbol{b}^{(L+1)},  \label{eq:ResNet_u}
\end{numcases}
where $\boldsymbol{H}^{(l)}$ is a linear mapping operator in the $l$-th hidden layer. When the dimensions of $y^{(l-1)}$ and $y^{(l)}$ are identical, $\boldsymbol{H}^{(l)}$ is an identical mapping. Otherwise, it is defined as $\boldsymbol{H}^{(l)}=\boldsymbol{W}^{(l, l-1)}_l \boldsymbol{y}^{(l-1)} + \boldsymbol{b}^{(l)}_l$, where $\boldsymbol{W}^{(l, l-1)}_l \in \mathbb{R}^{m_l \times m_{l-1}}$ and $\boldsymbol{b}^{(l)}_l \in \mathbb{R}^{m_l}$.

\subsubsection{Hidden-layer Concatenated ELM}
The Hidden-Layer Concatenated Extreme Learning Machine (HLConcELM), introduced in \cite{ni2023numerical}, is a novel method based on a modified feedforward neural network (HLConcFNN). This approach incorporates a logical concatenation of hidden layers, where the output layer neurons are formed by concatenating all neurons from the hidden layers. HLConcELM addresses the limitations of conventional ELM and demonstrates high accuracy in solving PDEs.

In this work, HLConcELM is assumed to have two hidden layers, as illustrated in Figure \ref{fig:HLConcELM}. The outputs of these hidden layers is concatenated by the concatenate operation $\textbf{c}$. Let the weights and biases in the hidden layers be $\boldsymbol{W}^{(l, l-1)} \in \mathbb{R}^{m_l \times m_{l-1}}$ and $\boldsymbol{b}^{(l)} \in \mathbb{R}^{m_l}$, where $l=1, 2$. The network formulation is given by
\begin{numcases}{}
    \boldsymbol{y}^{(l)} = \sigma(\boldsymbol{W}^{(l, l-1)} \boldsymbol{y}^{(l-1)} + \boldsymbol{b}^{(l)}), \ \ l = 1,2,     \label{eq:HLConcELM_hidden}    \\
    \Phi = [\phi_1, \cdots, \phi_M] %=concatenate(\boldsymbol{y}^{(1)}, 
    = \textbf{c}(\boldsymbol{y}^{(1)},
    \boldsymbol{y}^{(2)}),  \label{eq:HLConcELM_phi}
\end{numcases}
where $\Phi$ represents a set of basis functions, and $M=m_1+m_2$ denotes the total number of basis functions. Finally the function represented by HLConcELM is given by $u(\boldsymbol{x})=\Phi \cdot \boldsymbol{w}$, where $\boldsymbol{w}$ is the weight vector to be determined.

\begin{figure}[htp]
    \centering
    \includegraphics[width=0.8\textwidth]{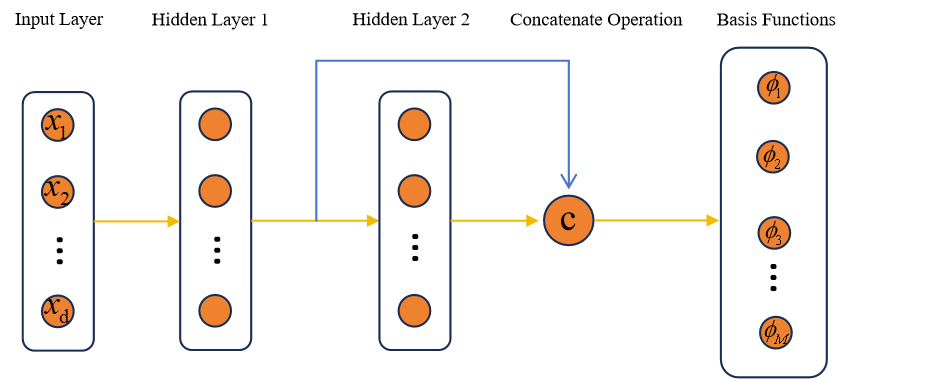}
    \caption{The architecture of HLConcELM \cite{ni2023numerical}, consisting of an input layer, two hidden layers, a concatenation operation, and an output layer.}
    \label{fig:HLConcELM}
\end{figure}

\subsection{The Neural Feature Space}
We conceptualize a neural network as a nonlinear feature mapping that transforms input vectors $\boldsymbol{x} \in \mathbb{R}^d$ into an output space $\mathbb{R}^M$, where $d$ denotes the dimensionality of the input space, and $M$ represents the number of outputs prior to the last layer, defining the dimensionality of the feature representation. From the perspective of approximation theory, the outputs of the neural network can be regarded as forming a basis that is globally supported in $\mathbb{R}^d$. Consequently, the neural feature space, denoted by $\mathcal{P}_{NN}$, is defined as the linear space spanned by the basis functions $\{\phi_j\}$ and each function corresponds to an output of the neural network:
\begin{equation}
    \label{eq:basis}
    \mathcal{P}_{NN} = \text{span}\{\phi_1, \phi_2, \cdots, \phi_M\}.
\end{equation}

Let $\mathcal{F}(\boldsymbol{x}; \boldsymbol{\theta})$ represent a nonlinear feature mapping realized by a neural network, where $\boldsymbol{x} \in \mathbb{R}^d$ is the input vector and $\boldsymbol{\theta}$ comprises the network parameters, including weights and biases. The basis functions can be defined as
\begin{equation}
    \label{eq:basis_f}
    \Phi = [\phi_1, \phi_2, \cdots, \phi_M] = \mathcal{F}(\boldsymbol{x}; \boldsymbol{\theta}).
\end{equation}
The function to be learned  is formulated as
\begin{equation}
    \label{eq:u}
    u(\boldsymbol{x}) = \sum_{i=1}^{M} w_i \phi_i = \Phi \cdot \boldsymbol{w}, 
\end{equation}
where $\boldsymbol{w} = [w_1, w_2, \cdots, w_M]^T$.

\subsection{Randomized Neural Networks}
Randomized neural networks refer to a class of neural networks in which all parameters, except those in the last layer, are randomly initialized and then remain fixed thereafter. In the work \cite{shang2023randomized, shang2024randomized, shang2023randomized2, wang2024randomized}, randomized neural networks have been employed to solve both linear and nonlinear PDEs. Notably, the ELM can be regarded as a special case of a single-hidden-layer randomized neural network.

We consider a neural network that maps a $d$-dimensional input vector $\boldsymbol{x}$ to an $M$-dimensional output represented by the basis functions $\Phi$, as defined in Equation~\eqref{eq:basis_f}. In this framework, the parameters $\boldsymbol{\theta}$ are randomly initialized and remain fixed throughout training, while the unknown parameters $\boldsymbol{w}$ are determined via the least squares method. The function being approximated is expressed as in Equation~\eqref{eq:u}. This equation encapsulates the relationship between the input $\boldsymbol{x}$ and the output $\Phi$, with $\boldsymbol{w}$ being the set of parameters that the network adjusts to minimize the approximation error.

To construct the system for solving $\boldsymbol{w}$, we define a real-valued matrix $\mathbf{A} \in \mathbb{R}^{N \times M}$ whose entries are evaluated from the basis functions $\Phi$ at discrete points in the dataset $S$:
\begin{equation}
    \label{eq:A_basis}
    \begin{aligned}
    \mathbf{A} = & \Phi(S) = (\mathcal{F}(\boldsymbol{x}_1; \boldsymbol{\theta}), \mathcal{F}(\boldsymbol{x}_2; \boldsymbol{\theta}), \cdots, \mathcal{F}(\boldsymbol{x}_N; \boldsymbol{\theta}))^T \\
    = &
    \begin{bmatrix}
        \phi_1(\boldsymbol{x}_1) & \phi_2(\boldsymbol{x}_1) & \cdots & \phi_M(\boldsymbol{x}_1) \\
        \phi_1(\boldsymbol{x}_2) & \phi_2(\boldsymbol{x}_2) & \cdots & \phi_M(\boldsymbol{x}_2) \\
        \vdots & \vdots & \cdots & \vdots \\
        \phi_1(\boldsymbol{x}_N) & \phi_2(\boldsymbol{x}_N) & \cdots & \phi_M(\boldsymbol{x}_N) \\
    \end{bmatrix},
    \end{aligned}
\end{equation}
where $\boldsymbol{x}_i \in S$ for $i=1, 2, \cdots, N$. 

The vector of coefficients $\boldsymbol{w}$ is obtained by solving the linear system $\mathbf{A} \cdot \boldsymbol{w} = \mathbf{F}$, where $\mathbf{F} = [f(\boldsymbol{x}_1), f(\boldsymbol{x}_2), \cdots, f(\boldsymbol{x}_N)]^T$ is the vector of function values at the discrete points. The least squares method is used to solve for $\boldsymbol{w}$, ensuring an optimal approximation of the target function. The complete procedure is summarized in Algorithm~\ref{algo:net_func_app}.

\begin{algorithm}[htp]
    \caption{Randomized neural networks for function approximation}
    \label{algo:net_func_app}
    Let $N$ be the number of points in the dataset $S$, and let $M$ be the number of basis functions, which also corresponds to the layer size in the neural network. \\
    \textbf{Step 1:} Initialize the weights of the network randomly and obtain the basis functions $\Phi$ from the network outputs. \\
    \textbf{Step 2:} Compute the matrix $\mathbf{A} = \Phi(S)$ as defined in Equation~\eqref{eq:A_basis}. \\
    \textbf{Step 3:} Define the approximation of the learned function as $u(\boldsymbol{x}) = \Phi \cdot \boldsymbol{w}$. \\
    \textbf{Step 4:} Construct the right-hand vector $\mathbf{F} = [f(\boldsymbol{x}_1),f(\boldsymbol{x}_2),\cdots,f(\boldsymbol{x}_N)]^T$. \\
    \textbf{Step 5:} Solve the linear system $\mathbf{A} \cdot \boldsymbol{w} = \mathbf{F}$ using the least squares method to find the coefficients $\boldsymbol{w}$. \\
    The function $u(\boldsymbol{x})$ is the approximation of the desired function by the network.
\end{algorithm}

% 方法与算法 Methods and Algorithms
\section{Methodologies}
\label{sec:methods}

\subsection{Neural Networks Based on Tensor Products}
In the architecture of randomized neural networks, various network types, such as ELM, MLP, ResNet, and HLConcELM, can be employed to represent the basis functions denoted by $\Phi$. However, we have observed that when the number of basis functions becomes large, constructing a sufficiently large neural network is often necessary. This is particularly true when solving PDEs, where the computation of partial derivatives for each basis function via automatic differentiation becomes required. Unfortunately, this process can be time-consuming. To address this challenge, inspired by the tensor neural networks introduced in \cite{wjx2023}, we propose the TPNet architecture. TPNet is designed to optimize the process of automatic differentiation and improve network efficiency, significantly reducing computational time.

Suppose that the input vector $\boldsymbol{x}$ is a $d$-dimensional vector. We construct two subnetworks, each processing the $d$-dimensional input independently. The functions represented by these two subnetworks are defined as follows:
\begin{equation}
    \label{eq:TP_basis_i}
    \Phi_k = (\phi_{k1}, \phi_{k2}, \cdots, \phi_{kp}) = \mathcal{F}_k(\boldsymbol{x}; \boldsymbol{\theta}_k), \quad k=1, 2,
\end{equation}
where $\Phi_k$ represents the set of functions generated by the $k$-th subnetwork, $\mathcal{F}_k(\boldsymbol{x}; \boldsymbol{\theta}_k)$ denotes the nonlinear feature mapping performed by the $k$-th subnetwork, and $p$ is the number of output features in each subnetwork. The architecture of the TPNet is illustrated in Figure~\ref{fig:TPNet}.

After obtaining two sets of functions from the subnetworks, we construct the set of basis functions $\Phi$, using the tensor product:
\begin{equation}
    \label{eq:TP_basis}
    \Phi = \Phi_1 \otimes \Phi_2,
\end{equation}
where $\otimes$ denotes the tensor product, defined as:
\begin{equation}
    \label{eq:TP}
    \Phi_1 \otimes \Phi_2 = \{ \phi_{1m} \phi_{2n} \mid 1 \le m, n \le p \}.
\end{equation}
Since each subnetwork produces $p$ functions, the total number of basis functions in $\Phi$ is $M=p^2$.

To compute the first-order partial derivative with respect to $x_i$, %apply the differentiation rule for tensor products:
we have
\begin{equation}
    \label{eq:TP_x}
    \begin{aligned}
    (\Phi)_{x_i} &= (\Phi_1)_{x_i} \otimes \Phi_2 + \Phi_1 \otimes (\Phi_2)_{x_i} \\
    &= \left\{ \frac{  \d  \phi_{1m}}{\d x_i} \phi_{2n} + \left.\phi_{1m} \frac{\d \phi_{2n}}{\d x_i} \right\vert 1 \le m, n \le p  \right\}.
    \end{aligned}
\end{equation}
where $(\Phi)_{x_i}$ represents the set of functions after differentiation with respect to $x_i$. Thus, to obtain the derivatives of the basis functions in $\Phi$, it is sufficient to compute the derivatives of the individual function in $\Phi_1$ and $\Phi_2$ for each variable.

In summary, as illustrated in Figure \ref{fig:TPNet}, TPNet processes the $d$-dimensional input variables through two subnetworks. Both subnetworks receive all 
$d$ input coordinates and produce $p$ outputs each. The complete set of basis functions is constructed by taking the tensor product of these two output vectors, yielding $M = p^2$ basis functions. The PDE solution is then expressed as a linear combination of these 
$M$ basis functions, forming the core concept of TPNet. It should be noted that the two subnetworks in TPNet can adopt any neural network architecture. In this work, we implement them using ELM, MLP, and ResNet architectures, yielding the variants TP-ELM, TP-MLP, and TP-ResNet, respectively. The expressivity of TPNet arises from two dynamic aspects: first, once initialized, the subnetwork parameters remain fixed during training, so the initialization strategy determines the characteristics of the basis functions; second, by generating $M= p^2$ basis functions via tensor product, TPNet can dynamically represent any target solution through an appropriate linear combination of these bases.

\begin{figure}[htp]
    \centering
    \includegraphics[width=0.8\textwidth]{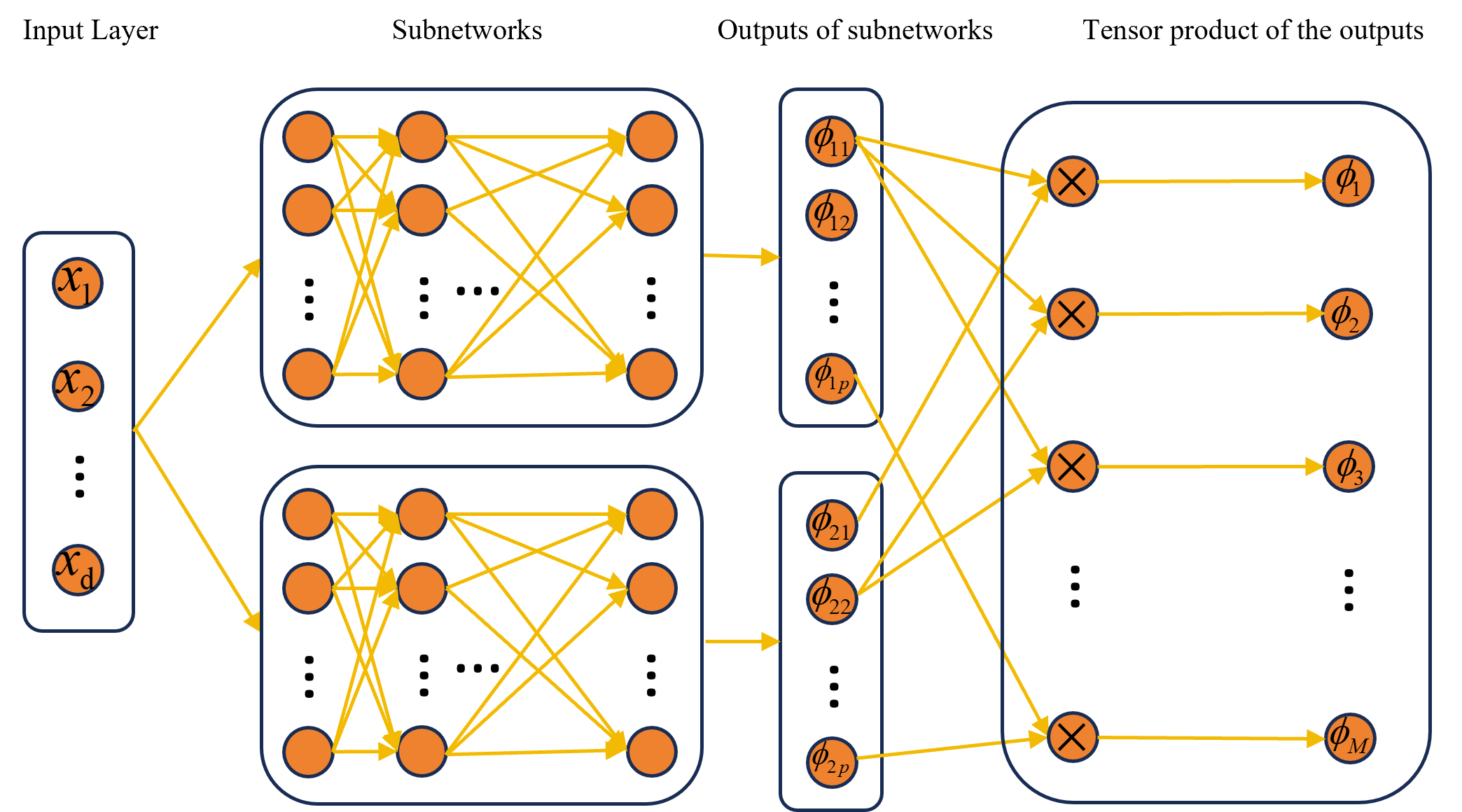}
    \caption{The architecture of TPNet consists of an input layer and two subnetworks. Each subnetwork processes $d$-dimensional input. The basis functions are obtained by taking the tensor product of the subnetwork outputs.}
    \label{fig:TPNet}
\end{figure}

\subsection{Linear Partial Differential Equations}
Consider the following generic linear PDE
\begin{numcases}{}
    \mathcal{L}u= f,    \  \mbox{in} \enspace \Omega,      \label{eq:control_equation}    \\
    \mathcal{B}u = g,   \  \mbox{on} \enspace \partial \Omega,  \label{eq:boundary_condition}
 \end{numcases}
where $u$ represents the scalar field function that we seek to determine, $\mathcal{L}$ is a linear PDE operator defined in the domain $\Omega$, $\mathcal{B}$ is a linear boundary operator defined on the boundary $\partial \Omega$, $f$ is the source term, and Equation \eqref{eq:boundary_condition} represents the boundary condition. Suppose that this linear PDE is well-posed and we employ the TPNet to solve such problems.% and to approximate $u$ as the solution.

We consider the basis functions $\Phi$ generated by the outputs of the TPNet. Applying the linear differential operators $\mathcal{L}$ and $\mathcal{B}$ to these basis functions yields the following expressions:
\begin{numcases}{}
    \mathcal{L}\Phi = (\mathcal{L}\phi_1, \mathcal{L}\phi_2, \cdots, \mathcal{L}\phi_M),      \label{eq:L_ctrl_eq}    \\
    \mathcal{B}\Phi = (\mathcal{B}\phi_1, \mathcal{B}\phi_2, \cdots, \mathcal{B}\phi_M).      \label{eq:B_boun_con}
\end{numcases}
Now, suppose our dataset $S$ consists of $N = N_r + N_b$ collocation points, where $N_r$ represents the number of points in the domain $\Omega$ and $N_b$ is the number of points on the boundary $\partial \Omega$. We divide the collocation points into two subsets: $S_r$ for points inside the domain and $S_b$ for points on the boundary. Using these collocation points, we construct the system matrix $\mathbf{A}$ as follows:
\begin{equation}
    \label{eq:linear_system_left}
    \begin{aligned}
    \mathbf{A} &= 
    \begin{bmatrix}
        \mathcal{L}\Phi(S_r) \\
        \mathcal{B}\Phi(S_b)
    \end{bmatrix} \\
    &= 
    \begin{bmatrix}
        \mathcal{L}\phi_1(\boldsymbol{x}_1) & \mathcal{L}\phi_2(\boldsymbol{x}_1) & \cdots & \mathcal{L}\phi_M(\boldsymbol{x}_1) \\
        \vdots & \vdots & \cdots  & \vdots \\
        \mathcal{L}\phi_1(\boldsymbol{x}_{N_r}) & \mathcal{L}\phi_2(\boldsymbol{x}_{N_r}) & \cdots & \mathcal{L}\phi_M(\boldsymbol{x}_{N_r}) \\
        \mathcal{B}\phi_1(\boldsymbol{x}_{N_r+1}) & \mathcal{B}\phi_2(\boldsymbol{x}_{N_r+1}) & \cdots & \mathcal{B}\phi_M(\boldsymbol{x}_{N_r+1}) \\
        \vdots & \vdots & \cdots  & \vdots \\
        \mathcal{B}\phi_1(\boldsymbol{x}_{N_r+N_b}) & \mathcal{B}\phi_2(\boldsymbol{x}_{N_r+N_b}) & \cdots & \mathcal{B}\phi_M(\boldsymbol{x}_{N_r+N_b}) \\
    \end{bmatrix}.
    \end{aligned}
\end{equation}

As discussed in the previous section, the vector of coefficients $\boldsymbol{w}$ can be determined by solving the linear system: $\mathbf{A} \cdot \boldsymbol{w} = \mathbf{F}$, where the right-hand-side vector $\mathbf{F}$ is defined as:
\begin{equation}
    \label{eq:linear_system_right}
    \boldsymbol{F} = [ f(\boldsymbol{x}_1),  \cdots, f(\boldsymbol{x}_{N_r}), g(\boldsymbol{x}_{N_r +1}), \cdots, g(\boldsymbol{x}_{N_r + N_b}) ]^T.
\end{equation}
Here, $f(\boldsymbol{x}_i)$ ($1 \le i \le N_r$) represents the source term evaluated at the collocation points in the domain $\Omega$ and $g(\boldsymbol{x}_i)$ ($N_r+1 \le i \le N_r+N_b$) represents the boundary condition term evaluated at the collocation points on the boundary $\partial \Omega$. The procedure for applying the TPNet to solve linear PDEs is summarized in Algorithm~\ref{algo:tp_linear_pde} and Figure \ref{fig:TPNet_linear}, which is also suitable for solving other kinds of PDEs.

\begin{algorithm}[htp]
    \caption{TPNet for solving the linear PDEs}
    \label{algo:tp_linear_pde}
    Let $N_r$ and $N_b$ denote the number of points in the domain $\Omega$ and on the boundary $\partial \Omega$, respectively, and let $M$ be the number of basis functions and $p$ be the number of outputs of a subnetwork. Construct the data sets $S_r$ for the domain and $S_b$ for the boundary.    \\
    \textbf{Step 1:} Randomly initialize the weights of the subnetworks and obtain the basis functions $\Phi=\Phi_1 \otimes \Phi_2$. \\
    \textbf{Step 2:} Compute the matrix $\mathbf{A} = [\mathcal{L}\Phi(S_r), \mathcal{B}\Phi(S_b)]^T$ in Equation~\eqref{eq:linear_system_left} using the basis functions evaluated at the collocation points in the data sets $S_r$ and $S_b$. \\
    \textbf{Step 3:} Express the learned function as $u(\boldsymbol{x}) = \Phi \cdot \boldsymbol{w}$, where $\boldsymbol{w}$ is the weight vector to be determined. \\
    \textbf{Step 4:} Construct the right-hand vector $\mathbf{F} = [ f(\boldsymbol{x}_1), \cdots, f(\boldsymbol{x}_{N_r}), g(\boldsymbol{x}_{N_r +1}), \cdots, g(\boldsymbol{x}_{N_r + N_b}) ]^T$. \\
    \textbf{Step 5:} Solve the linear system $\mathbf{A} \cdot \boldsymbol{w} = \mathbf{F}$ using least squares method to find the coefficients $\boldsymbol{w}$. \\
    The solution $u(\boldsymbol{x})$ is  approximated by the TPNet.% providing a solution to the linear PDEs.
\end{algorithm}

\begin{figure}[htp]
    \centering
    \includegraphics[width=0.8\textwidth]{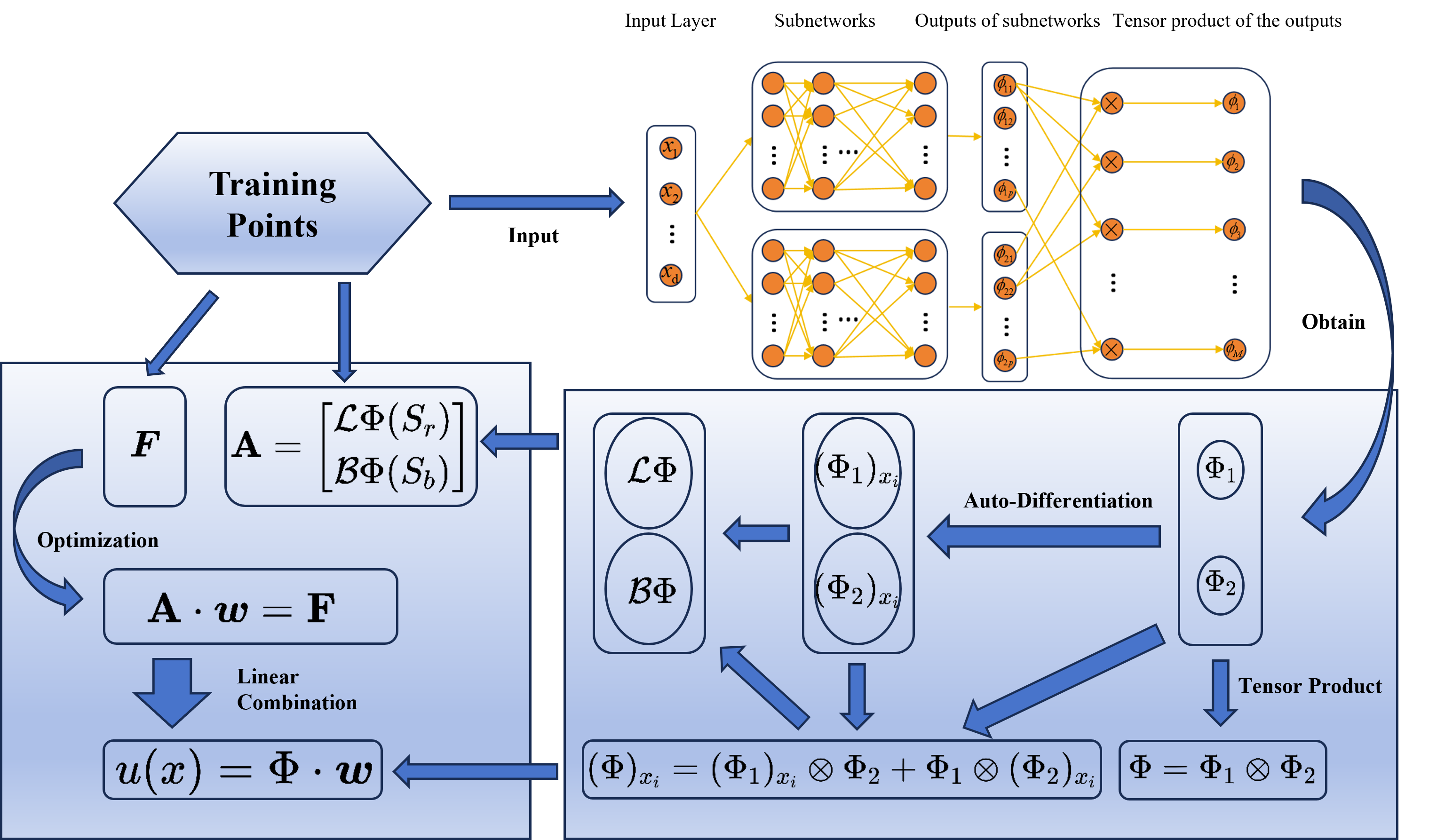}
    \caption{The process of TPNet for solving the linear PDEs.}
    \label{fig:TPNet_linear}
\end{figure}

\subsection{Block Time-Marching for Long-Time Simulations}
In long-time simulations, it has been observed that TPNet may face challenges in accurately approximating solutions over extended periods. As time progresses, the accuracy of the solutions may loss. To address this issue and improve TPNet’s performance in long-time simulations, we propose adopting the block time-marching strategy \cite{dong2021local}. This approach involves decomposing the simulation into smaller, manageable time steps or blocks, thereby incrementally updating the solution at each stage. By systematically propagating the solution through smaller intervals, this strategy improves accuracy and reduces error accumulation, thereby enhancing the solution's precision and stability during extended simulations, which have been validated in numerical experiments in Section \ref{sec:BTM}.

To solve time-dependent PDEs, we consider the following problem in $\Omega \times (0, t_f]$:
\begin{equation}
    \label{eq:time_equation}
    \begin{array}{r@{}l}
        \left\{
        \begin{aligned}
            \mathcal{L}u(\boldsymbol{x}, t) &= f(\boldsymbol{x}, t), &  & \mbox{in} \enspace \Omega \times (0, t_f],          \\
            \mathcal{B}u(\boldsymbol{x}, t) &= g(\boldsymbol{x}, t),        &  & \mbox{on} \enspace \partial \Omega \times [0, t_f], \\
            \mathcal{I}u(\boldsymbol{x}, 0) &= h(\boldsymbol{x}, 0),            &  & \mbox{in} \enspace \Omega,
        \end{aligned}
        \right.
    \end{array}
\end{equation}
where $\mathcal{L}$ and $\mathcal{B}$ are linear differential operators, $\mathcal{I}$ represents the initial condition operator.

To efficiently solve the problem \eqref{eq:time_equation}, we adopt a block time-marching strategy. The time step size is defined as $dt = \frac{t_f}{D_t}$ and $D_t$ is the number of time blocks. The problem is reformulated as a sequence of time-block problems:
\begin{equation}
    \label{eq:time_block_marching}
    \begin{array}{r@{}l}
        \left\{
        \begin{aligned}
            \mathcal{L}u(\boldsymbol{x}, t) &= f(\boldsymbol{x}, t), &  & \mbox{in} \enspace \Omega \times (t_k, t_{k+1}],          \\
            \mathcal{B}u(\boldsymbol{x}, t) &= g(\boldsymbol{x}, t),        &  & \mbox{on} \enspace \partial \Omega \times [t_k, t_{k+1}], \\
            \mathcal{I}u(\boldsymbol{x}, t_k) &= h(\boldsymbol{x}, t_k),            &  & \mbox{in} \enspace \Omega.
        \end{aligned}
        \right.
    \end{array}
\end{equation}
where, $0 \leq k \leq D_t - 1$, $t_k = k \cdot dt$ represents the start time of the $k$-th time block, and the function $h(\boldsymbol{x}, t_k)$ serves as the initial condition for the PDEs at time $t = t_k$ in each time block. By iteratively solving these $D_t$ initial-boundary value problems using the TPNet, we obtain the solution $u(\boldsymbol{x}, t)$ across the entire spatial-temporal domain.

\subsection{Nonlinear Partial Differential Equations}
The procedures of the TPNet algorithm for nonlinear PDEs are fundamentally analogous to those used for linear PDEs. However, a key distinction lies in the iterative process employed to find a convergent weight vector $\boldsymbol{w}$, which enhances the neural network's ability to approximate the exact solution within a desired level of accuracy, which is commonly referred to as the Picard iteration and involves repeatedly applying the neural network with updated weights until the solution converges to a steady state. Each iteration refines the approximation by using the output from the previous iteration as input for the next. By leveraging the Picard iteration, the TPNet can effectively tackle the complexities inherent in nonlinear PDEs, progressively improving the solution approximation with each iteration. The Picard iteration method, also known as the fixed-point iteration method, is an iterative technique for solving nonlinear equations \cite{suli2003introduction}. In this work, we focus on nonlinear PDEs of the form given in Equation \eqref{eq:nonlinear_pde}, which can be solved using the Picard iteration method. This approach has been successfully utilized for solving nonlinear PDEs in RNN-PG \cite{shang2023randomized}.

We consider a class of nonlinear PDE of the form:
\begin{equation}
    \label{eq:nonlinear_pde}
    \begin{array}{r@{}l}
        \left\{
        \begin{aligned}
            \mathcal{L}u(\boldsymbol{x}) + \mathcal{N}[u(\boldsymbol{x})] &= f(\boldsymbol{x}), &  & \mbox{in} \enspace \Omega,          \\
            \mathcal{B}u(\boldsymbol{x}) &= g(\boldsymbol{x}),        &  & \mbox{on} \enspace \partial \Omega. \\
        \end{aligned}
        \right.
    \end{array}
\end{equation}
where $\mathcal{L}$ represents a linear differential operator, $\mathcal{N}$ is a nonlinear operator, %$f(\boldsymbol{x})$ and $g(\boldsymbol{x})$ are the source and boundary terms, respectively, 
and $\mathcal{B}$ is a linear boundary operator.

Given an initial vector of coefficients $\boldsymbol{w}_0$, we approximate the initial solution  as $u_0(\boldsymbol{x}) = \Phi \cdot \boldsymbol{w}_0$, where $\Phi$ denotes the basis functions generated by the neural network. Using this initial solution, the nonlinear term $\mathcal{N}[u_0(\boldsymbol{x})]$ can be evaluated, thereby rendering it a known quantity in the problem. Then the PDE reduces to a linearized equation:
\begin{equation}
    \label{eq:nonlinear_linear_pde}
    \begin{array}{r@{}l}
        \left\{
        \begin{aligned}
            \mathcal{L}u(\boldsymbol{x}) &= f(\boldsymbol{x}) - \mathcal{N}[u_0(\boldsymbol{x})], &  & \mbox{in} \enspace \Omega,          \\
            \mathcal{B}u(\boldsymbol{x}) &= g(\boldsymbol{x}),        &  & \mbox{on} \enspace \partial \Omega. \\
        \end{aligned}
        \right.
    \end{array}
\end{equation}
This linearized equation can be efficiently solved using the TPNet method. By iterating this process, we obtain a convergent coefficient vector $\boldsymbol{w}$, and the final solution to the original nonlinear PDEs is approximated as $u(\boldsymbol{x}) = \Phi \cdot \boldsymbol{w}$.

The procedure for solve nonlinear PDEs using TPNet is summarized in Algorithm~\ref{algo:tp_nonlinear_pde}, which iteratively updates the coefficients until convergence.

\begin{algorithm}[htp]
    \caption{TPNet for solving the nonlinear PDEs}
    \label{algo:tp_nonlinear_pde}
    Define $N_r$ and $N_b$ as the number of points in the domain $\Omega$ and on the boundary $\partial \Omega$, respectively, and let $M$ be the number of basis functions and $p$ be the number of outputs of a subnetwork. Construct the data sets $S_r$ and $S_b$ for the domain and boundary, respectively. \\    
    Let $k_{max}$ be the maximum number of iterations, $\epsilon$ be the convergence tolerance, and $\boldsymbol{w}_0$ be the initial vector of coefficients. \\
    \textbf{Step 1:} Randomly initialize the weights of the subnetworks and obtain the basis functions $\Phi=\Phi_1 \otimes \Phi_2$. \\
    \textbf{Step 2:} Compute the matrix $\mathbf{A} = [\mathcal{L}\Phi(S_r), \mathcal{B}\Phi(S_b)]^T$ according to Equation~\eqref{eq:linear_system_left}. \\
    \textbf{Step 3:} \\
    \For{$k=0, 1, 2, \cdots, k_{max}-1$}{
        Express the current approximation of the solution as $u_k(\boldsymbol{x}) = \Phi \cdot \boldsymbol{w}_k$. \\
        Compute the nonlinear term $\mathcal{N}[u_k(\boldsymbol{x})]$. \\
        Construct the right-hand vector $\mathbf{F} = [ f(\boldsymbol{x}_1)-\mathcal{N}[u_k(\boldsymbol{x}_1)], \cdots, f(\boldsymbol{x}_{N_r})-\mathcal{N}[u_k(\boldsymbol{x}_{N_r})], g(\boldsymbol{x}_{N_r +1}), \cdots, g(\boldsymbol{x}_{N_r + N_b}) ]^T$, incorporating the nonlinear term and boundary conditions. \\
        Solve the linear system $\mathbf{A} \cdot \boldsymbol{w}_{k+1} = \mathbf{F}$ using least squares method to obtain the updated vector of coefficients $\boldsymbol{w}_{k+1}$. \\
       \If{$\lVert \boldsymbol{w}_k - \boldsymbol{w}_{k+1} \rVert < \epsilon \space$}{
            $\boldsymbol{w}=\boldsymbol{w}_{k_{k+1}}$. \\
           The iteration is halted as the stopping criterion is met. \\
        }
    }
    %Upon completion of the iterations, 
    The vector of coefficients $\boldsymbol{w}$ is obtained, %and the solution to the nonlinear PDEs is approximated 
    and the solution is %by the TPNet as 
    $u(\boldsymbol{x}) = \Phi \cdot \boldsymbol{w}$.
\end{algorithm}

\subsection{Comparisons with Other Algorithms}
The training approaches differ significantly. In existing algorithms such as PINN \cite{PINN}, DGM \cite{sirignano2018dgm}, and DRM \cite{yu2018deep}, gradient descent optimizers (e.g., Adam \cite{kingma2014adam} and L-BFGS \cite{liu1989limited}) are primarily employed. By contrast, the TPNet proposed in this work adopts a least squares framework. In this method, the neural network’s parameters are fixed after initialization, and only the coefficients of the linear combination in the output layer --- determined solely via least squares --- require optimization.

The most significant distinction lies in the absence of domain decomposition and the adoption of tensor products to construct basis functions. Algorithms such as RFM \cite{chen2024optimization, chenel2023, chen2022bridging} and RNN-PG rely on techniques like partition of unity (PoU) and mesh generation for domain decomposition. While these methods enable highly accurate solutions for PDEs, they are computationally intensive to implement. Furthermore, when the number of basis functions increases, HLConcELM \cite{ni2023numerical} incurs significant computational costs for derivative calculations via automatic differentiation, thereby reducing efficiency. In contrast, TPNet leverages tensor products to construct basis functions. The derivatives of these basis functions are computed as the tensor product of derivatives from two independent function sets, dramatically improving computational efficiency.

\subsection{Approximation Theory of TPNet}
In this section, to further explore TPNet's approximation capabilities, we assume that the ELM is employed in each subnetwork in the TPNet framework, which is called TP-ELM.
The approximate solution can be represented by the following formula
\begin{equation}
    \label{eq:TP_ELM}
    u_M(\boldsymbol{x}) =  \sum_{i=1}^{p} \sum_{j=1}^{p} w_{ij} \sigma(\boldsymbol{W}_1^i \boldsymbol{x} + \boldsymbol{b}_1^i) \sigma(\boldsymbol{W}_2^j \boldsymbol{x} + \boldsymbol{b}_2^j), \ \forall \boldsymbol{x} \in \Omega \subset \mathbb{R}^d,
\end{equation}
where $\boldsymbol{W}_k \in \mathbb{R}^{p \times d}$, $\boldsymbol{b}_k \in \mathbb{R}^{p}$ and $k=1,2$, which means that the weights and biases come from $k$-th subnetwork, the $\boldsymbol{W}_k^i$ is the $i$-th row in $\boldsymbol{W}_k$ and $\boldsymbol{b}_k^i$ is the $i$-th row in $\boldsymbol{b}_k$, the $w_{ij}$ is the unknown solved by the least squares method. 

\begin{theorem} 
    \label{thm:u_C}
    Let $\sigma$ be any continuous sigmoidal function. The finite sums of
    \begin{equation}
        \label{eq:u_C}
        u_{\mathcal{C}}(\boldsymbol{x}) =  \sum_{i=1}^{M} w_i \sigma(\boldsymbol{W}^i \boldsymbol{x} + \boldsymbol{b}^i)
    \end{equation}
    are dense in $\mathcal{C}(\Omega)$, where $\boldsymbol{W} \in \mathbb{R}^{M \times d}$, $\boldsymbol{b} \in \mathbb{R}^{M}$. In other words, given any $f \in \mathcal{C}(\Omega)$ and $\epsilon > 0$, there exists a sum $u_{\mathcal{C}} (\boldsymbol{x})$, such that
    \begin{equation}
        \label{eq:u_C_to_epsilon}
        \lvert u_{\mathcal{C}}(\boldsymbol{x}) - f(\boldsymbol{x}) \rvert < \epsilon.
    \end{equation}
\end{theorem}
\begin{proof}
    It is a direct consequence in reference \cite[Theorem 2]{cybenko1989approximation}.
\end{proof}

\begin{theorem}
    \label{thm:u_C_tanh}
    When using $\tanh$ as the activation function in \eqref{eq:u_C}, the finite sums of \eqref{eq:u_C} are dense in $\mathcal{C}(\Omega)$.
\end{theorem}
\begin{proof}
    Consider a continuous sigmoid function $\tilde{\sigma}(x)=\frac{1}{1+e^{-x}}$, then $\tanh(x)=2\tilde{\sigma}(2x)-1$. From Theorem \ref{thm:u_C}, for any given $f\in \mathcal{C}(\Omega)$ and $\epsilon >0$, there is a sum $\tilde{u}_{\mathcal{C}} = \sum_{i=1}^{\tilde{M}} \tilde{w}_i \tilde{\sigma}(\tilde{\boldsymbol{W}}^i \boldsymbol{x} + \tilde{\boldsymbol{b}}^i)$, such that
    \begin{equation*}
        \lvert \tilde{u}_{\mathcal{C}}(\boldsymbol{x}) - f(\boldsymbol{x}) \rvert < \frac{\epsilon}{2}.
    \end{equation*}
We denote $\hat{w}=\sum_{i=1}^{\tilde{M}}\tilde{w}_i$. If $\hat{w}=0$, 
the finite sum $u_{\mathcal{C}}(\boldsymbol{x})=\sum_{i=1}^{\tilde{M}} \frac{1}{2} \tilde{w}_i \tanh (\frac{1}{2} (\tilde{\boldsymbol{W}}^i \boldsymbol{x} + \tilde{\boldsymbol{b}}^i))$ satisfies
    \begin{equation*}
        \lvert u_{\mathcal{C}}(\boldsymbol{x}) - f(\boldsymbol{x}) \rvert = \lvert \sum_{i=1}^{\tilde{M}} \frac{1}{2} \tilde{w}_i(2\tilde{\sigma}(\tilde{\boldsymbol{W}}^i \boldsymbol{x} + \tilde{\boldsymbol{b}}^i) - 1) - f(\boldsymbol{x}) \rvert = \lvert \tilde{u}_{\mathcal{C}}(\boldsymbol{x}) - f(\boldsymbol{x}) \rvert < \frac{\epsilon}{2} < \epsilon.
    \end{equation*} 
If $\hat{w} \neq 0$, since $\Omega \subset \mathbb{R}^d$ is compact and thus is bounded and closed, there exists $\tilde{\boldsymbol{W}}^{r} \in \mathbb{R}^{1 \times d}$ and $\tilde{\boldsymbol{b}}^{r} \in \mathbb{R}$ such that
    \begin{equation*}
        \lvert \tanh(\tilde{\boldsymbol{W}}^{r} \boldsymbol{x} + \tilde{\boldsymbol{b}}^{r}) - 1 \rvert < \frac{\epsilon}{\lvert \hat{w} \rvert}
    \end{equation*}
    holds for all $\boldsymbol{x} \in \Omega$. Therefore, the finite sum $u_{\mathcal{C}}(\boldsymbol{x})=\sum_{i=1}^{\tilde{M}} \frac{1}{2} \tilde{w}_i \tanh (\frac{1}{2} (\tilde{\boldsymbol{W}}^i \boldsymbol{x} + \tilde{\boldsymbol{b}}^i)) + \frac{1}{2}\hat{w}\tanh(\tilde{\boldsymbol{W}}^{r} \boldsymbol{x} + \tilde{\boldsymbol{b}}^{r})$ satisfies
    \begin{equation*}
        \begin{aligned}
        \lvert u_{\mathcal{C}}(\boldsymbol{x}) - f(\boldsymbol{x}) \rvert &= \lvert \sum_{i=1}^{\tilde{M}} \frac{1}{2} \tilde{w}_i(2\tilde{\sigma}(\tilde{\boldsymbol{W}}^i \boldsymbol{x} + \tilde{\boldsymbol{b}}^i) - 1) + \frac{1}{2}\hat{w}\tanh(\tilde{\boldsymbol{W}}^{r} \boldsymbol{x} + \tilde{\boldsymbol{b}}^{r}) - f(\boldsymbol{x}) \rvert \\
        & \leq \lvert \tilde{u}_{\mathcal{C}}(\boldsymbol{x}) - f(\boldsymbol{x}) \rvert +  \frac{1}{2}\lvert \hat{w} \rvert \lvert \tanh(\tilde{\boldsymbol{W}}^{r} \boldsymbol{x} + \tilde{\boldsymbol{b}}^{r}) - 1 \rvert \\
        & < \frac{\epsilon}{2} + \frac{\epsilon}{2} = \epsilon            
        \end{aligned}
    \end{equation*} 
    for all $\boldsymbol{x} \in \Omega$.
\end{proof}

For TP-ELM, we expand the definition of discriminatory functions in \cite{cybenko1989approximation} as follows

\begin{definition}
    $\sigma$ is said to be discriminatory if for a measure $\mu \in \mathcal{M}(\Omega)$
    \begin{equation}
        \int_{\Omega} \sigma(\boldsymbol{W}_1^r \boldsymbol{x} + \boldsymbol{b}_1^r) \sigma(\boldsymbol{W}_2^r \boldsymbol{x} + \boldsymbol{b}_2^r) \d \mu(\boldsymbol{x}) = 0,
    \end{equation}
    for all $\boldsymbol{W}_k^r \in \mathbb{R}^{1 \times d}$ and $\boldsymbol{b}_k^r \in \mathbb{R}$, where $k=1,2$, implies $\mu = 0$.
\end{definition}
As in \cite{cybenko1989approximation, chenel2023}, we have the following lemma and Theorem.
\begin{lemma} \cite{cybenko1989approximation, chenel2023} Any bounded, measurable sigmoidal function $\sigma$ is discriminatory. In particular, any continuous sigmoidal function is discriminatory.
\end{lemma}

\begin{theorem} \cite{cybenko1989approximation, chenel2023}
    \label{thm:TP_ELM}
    Let $\sigma$ be any continuous sigmoidal function. The finite sums in Equation \eqref{eq:TP_ELM} are dense in $\mathcal{C}(\Omega)$.
\end{theorem}

\begin{theorem}
    When using $\tanh$ as the activation function in \eqref{eq:TP_ELM}, the finite sums of the form
    \begin{equation}
        \label{eq:TP_ELM_tanh}
        u_M(\boldsymbol{x}) =  \sum_{i=1}^{p} \sum_{j=1}^{p} w_{ij} \tanh(\boldsymbol{W}_1^i \boldsymbol{x} + \boldsymbol{b}_1^i) \tanh(\boldsymbol{W}_2^j \boldsymbol{x} + \boldsymbol{b}_2^j)
    \end{equation}
    are dense in $\mathcal{C}(\Omega)$.
\end{theorem}
\begin{proof}
    The proof is similar to that of Theorem \ref{thm:u_C_tanh}. Consider a continuous sigmoidal function $\tilde{\sigma}(x)=\frac{1}{1+e^{-x}}$, then $\tanh(x)=2\tilde{\sigma}(2x)-1$. From Theorem \ref{thm:TP_ELM}, there is a sum $\tilde{u}_{\mathcal{C}} = \sum_{i=1}^{\tilde{p}} \sum_{j=1}^{\tilde{p}} \tilde{w}_{ij} \tilde{\sigma}(\tilde{\boldsymbol{W}}_1^i \boldsymbol{x} + \tilde{\boldsymbol{b}}_1^i) \tilde{\sigma}(\tilde{\boldsymbol{W}}_2^j \boldsymbol{x} + \tilde{\boldsymbol{b}}_2^j)$, such that
    \begin{equation*}
        \lvert \tilde{u}_{\mathcal{C}}(\boldsymbol{x}) - f(\boldsymbol{x}) \rvert < \frac{2 \epsilon}{5}.
    \end{equation*}

    Denote $\overline{w}=\sum_{i=1}^{\tilde{p}}\sum_{j=1}^{\tilde{p}}\lvert \tilde{w}_{ij} \rvert$,  $\tilde{t}_1 = \tanh (\tilde{\boldsymbol{W}}_1^r \boldsymbol{x} + \tilde{\boldsymbol{b}}_1^r)$ and $\tilde{t}_2 = \tanh (\tilde{\boldsymbol{W}}_2^r \boldsymbol{x} + \tilde{\boldsymbol{b}}_2^r)$. Since $\Omega \subset \mathbb{R}^d$ is compact and thus is bounded and closed, there exists $\tilde{\boldsymbol{W}}_1^{r} \in \mathbb{R}^{1 \times d}$, $\tilde{\boldsymbol{b}}_1^{r} \in \mathbb{R}$, $\tilde{\boldsymbol{W}}_2^{r} \in \mathbb{R}^{1 \times d}$ and $\tilde{\boldsymbol{b}}_2^{r} \in \mathbb{R}$ such that $ \lvert \overline{w} \rvert \lvert \tilde{t}_1 - 1 \rvert < \frac{2 \epsilon}{5}$, $ \lvert \overline{w} \rvert \lvert \tilde{t}_2 - 1 \rvert < \frac{2 \epsilon}{5}$ and $ \lvert \overline{w} \rvert \lvert 1 - \tilde{t}_1 - \tilde{t}_2 + \tilde{t}_1 \tilde{t}_2 \rvert = \lvert \overline{w} \rvert \lvert (1 - \tilde{t}_1) + \tilde{t}_2 (\tilde{t}_1 - 1) \rvert < 2 \lvert \overline{w} \rvert \lvert \tilde{t}_1 - 1 \rvert < \frac{4 \epsilon}{5}$.
 
    Then the finite sum 
    \begin{equation*}
        \begin{aligned}
        u_{\tilde{M}}(\boldsymbol{x})&=\sum_{i=1}^{\tilde{p}}\sum_{j=1}^{\tilde{p}} \frac{1}{4} \tilde{w}_{ij} \tanh (\frac{1}{2} (\tilde{\boldsymbol{W}}_1^i \boldsymbol{x} + \tilde{\boldsymbol{b}}_1^i)) \tanh (\frac{1}{2} (\tilde{\boldsymbol{W}}_2^j \boldsymbol{x} + \tilde{\boldsymbol{b}}_2^j)) \\
        & + \sum_{i=1}^{\tilde{p}}\sum_{j=1}^{\tilde{p}} \frac{1}{4} \tilde{w}_{ij} \tilde{t}_1 \tanh (\frac{1}{2} (\tilde{\boldsymbol{W}}_2^j \boldsymbol{x} + \tilde{\boldsymbol{b}}_2^j))\\
        & + \sum_{i=1}^{\tilde{p}}\sum_{j=1}^{\tilde{p}} \frac{1}{4} \tilde{w}_{ij} \tilde{t}_2 \tanh (\frac{1}{2} (\tilde{\boldsymbol{W}}_1^i \boldsymbol{x} + \tilde{\boldsymbol{b}}_1^i)) 
        + \sum_{i=1}^{\tilde{p}}\sum_{j=1}^{\tilde{p}} \frac{1}{4} \tilde{w}_{ij} \tilde{t}_1 \tilde{t}_2
        \end{aligned}
    \end{equation*}
    satisfies
    \begin{equation*}
        \begin{aligned}
        \lvert u_{\tilde{M}}(\boldsymbol{x}) - f(\boldsymbol{x}) \rvert &= \lvert \sum_{i=1}^{\tilde{p}} \sum_{j=1}^{\tilde{p}} \tilde{w}_{ij} \tilde{\sigma}(\tilde{\boldsymbol{W}}_1^i \boldsymbol{x} + \tilde{\boldsymbol{b}}_1^i) \tilde{\sigma}(\tilde{\boldsymbol{W}}_2^j \boldsymbol{x} + \tilde{\boldsymbol{b}}_2^j) - f(\boldsymbol{x})  \\
          & \quad + \sum_{i=1}^{\tilde{p}} \sum_{j=1}^{\tilde{p}} \tilde{w}_{ij} (\frac{1}{4} - \frac{1}{2} \tilde{\sigma}(\tilde{\boldsymbol{W}}_1^i \boldsymbol{x} + \tilde{\boldsymbol{b}}_1^i) - \frac{1}{2} \tilde{\sigma}(\tilde{\boldsymbol{W}}_2^j \boldsymbol{x} + \tilde{\boldsymbol{b}}_2^j)) \\ 
          & \quad+ \sum_{i=1}^{\tilde{p}}\sum_{j=1}^{\tilde{p}} \frac{1}{2} \tilde{w}_{ij} \tilde{t}_1 (\tilde{\sigma} (\tilde{\boldsymbol{W}}_2^j \boldsymbol{x} + \tilde{\boldsymbol{b}}_2^j) - \frac{1}{2}) \\
          & \quad+ \sum_{i=1}^{\tilde{p}}\sum_{j=1}^{\tilde{p}} \frac{1}{2} \tilde{w}_{ij} \tilde{t}_2 (\tilde{\sigma} (\tilde{\boldsymbol{W}}_1^i \boldsymbol{x} + \tilde{\boldsymbol{b}}_1^i) - \frac{1}{2})\\
        & \quad + \sum_{i=1}^{\tilde{p}}\sum_{j=1}^{\tilde{p}} \frac{1}{4} \tilde{w}_{ij} \tilde{t}_1 \tilde{t}_2  \rvert \\
%        \end{aligned}
 %       \end{equation*}
 %       \begin{equation*}
%        \begin{aligned}
        &= \lvert (\tilde{u}_{\mathcal{C}}(\boldsymbol{x}) - f(\boldsymbol{x})) + (\sum_{i=1}^{\tilde{p}}\sum_{j=1}^{\tilde{p}} \frac{1}{2} \tilde{w}_{ij} (\tilde{t}_1 - 1)\tilde{\sigma} (\tilde{\boldsymbol{W}}_2^j \boldsymbol{x} + \tilde{\boldsymbol{b}}_2^j)) \\
        & \quad + (\sum_{i=1}^{\tilde{p}}\sum_{j=1}^{\tilde{p}} \frac{1}{2} \tilde{w}_{ij} (\tilde{t}_2 - 1)\tilde{\sigma} (\tilde{\boldsymbol{W}}_1^i \boldsymbol{x} + \tilde{\boldsymbol{b}}_1^i)) \\
        & \quad + (\sum_{i=1}^{\tilde{p}}\sum_{j=1}^{\tilde{p}} \frac{1}{4} \tilde{w}_{ij}(1 - \tilde{t}_1 - \tilde{t}_2 + \tilde{t}_1 \tilde{t}_2)) \rvert \\
        &< \lvert \tilde{u}_{\mathcal{C}}(\boldsymbol{x}) - f(\boldsymbol{x}) \rvert + \frac{1}{2} \overline{w} \lvert \tilde{t}_1 - 1 \rvert + \frac{1}{2} \overline{w} \lvert \tilde{t}_2 - 1 \rvert \\
        &\quad + \frac{1}{4} \overline{w} \lvert 1 - \tilde{t}_1 - \tilde{t}_2 + \tilde{t}_1 \tilde{t}_2  \rvert \\
        &< \frac{2 \epsilon}{5} + \frac{\epsilon}{5} + \frac{\epsilon}{5} + \frac{\epsilon}{5} = \epsilon,
        \end{aligned}
    \end{equation*}
    for all $\boldsymbol{x} \in \Omega$.
\end{proof}

% 实验 Experiments
\section{Numerical Experiments}
\label{sec:experiments}
\subsection{Experimental Setup}
In this section, we present numerical experiments to evaluate the applicability and accuracy of TPNet. These experiments include function approximation and solving PDEs. For all experiments, we utilize ELM, MLP and ResNet as subnetworks within TPNet, resulting in three variants: TP-ELM, TP-MLP and TP-ResNet. For comparison, we also employ HLConcELM \cite{ni2023numerical} to solve the corresponding problems. All experiments are conducted on a high-performance server running CentOS 7, equipped with an Intel Xeon Platinum 8358 CPU (2.60 GHz) and an NVIDIA A100 GPU with 80GB of memory. Unless otherwise specified, we use the hyperbolic tangent ($\tanh$) activation function for all neural networks.

To assess the impact of the number of basis functions on the approximate solution, we select values for $M$ (number of basis functions) as $100$, $400$, $900$, $\cdots$, $10,000$ and for $p$ (number of outputs in each subnetwork) as $10$, $20$, $30$, $\cdots$, $100$. This systematic evaluation allows us to analyze the effects of $M$ and $p$ on the performance of HLConcELM, TP-ELM, TP-MLP and TP-ResNet. The architecture of HLConcELM can be succinctly described as [$d$, $\frac{M}{2}$, $\frac{M}{2}$]. In contrast, within TPNet, the ELM subnetwork follows the structure [$d$, $p$], for the MLP subnetwork is [$d$, $p$, $p$, $p$, $p$], while both the MLP and ResNet subnetworks follow the structure [$d$, $p$, $p$, $p$, $p$]. Notably, the concatenate operation in HLConcELM and the tensor product operation in TPNet are not explicitly represented in these descriptions. In the subsequent numerical experiments, the datasets used for training the neural networks primarily comprise collocation points within the computational domain and the corresponding values of the PDEs’ prescribed functions at these points.

To rigorously quantify the approximation capabilities of neural networks in our numerical experiments, we introduce two error metrics: the maximum absolute $L_{\infty}$ error and the $L_2$ error. These are defined as follows:
\begin{eqnarray}\label{eq:definitionerror}
L_{\infty} \enspace \mbox{error}  = \max_{1\leq i \leq N} \vert u_M(\boldsymbol{x}_i)-u_{exact}(\boldsymbol{x}_i) \vert,  \\
L_2 \enspace \mbox{error} = \left(\sum_{i=1}^{N}(u_M(\boldsymbol{x}_i)-u_{exact}(\boldsymbol{x}_i))^2\right)^{\frac{1}{2}},
\end{eqnarray}
where $u_M$ and $u_{exact}$ denote the network solution and the exact solution, respectively, and $\boldsymbol{x}_i$ ($1\le i \le N$) denote the evaluation points.

\subsection{Two-Dimensional Function Approximation}
To evaluate the function approximation capability of TPNet, we apply it to the following two-dimensional function defined in $\Omega=[-1, 1]^2$:
\begin{equation}
    \label{eq:2D_function}
    u(x, y) = \sin(\pi x) \sin(4\pi y).
\end{equation}

For this experiment, we use a uniform grid of $N_x \times N_y = 101 \times 101$ collocation points for training, where $N_x$ and $N_y$ denote the number of points along the $x$-axis and $y$-axis, respectively. All neural networks in this study are initialized using the Kaiming initialization method \cite{He_2015_ICCV}. We compare the performance of HLConcELM, TP-ELM, TP-MLP, and TP-ResNet by analyzing their $L_{\infty}$ errors, $L_{2}$ errors, and training time, as shown in Figure \ref{fig:error_2d} and Table \ref{tab:error_2d}. From Figure \ref{fig:error_2d}, we observe that as the number of basis functions increases, the error decreases for all network models. However, when the number of basis functions exceeds $M > 1,600$, we find that TP-ELM exhibits slightly higher errors than HLConcELM, TP-MLP achieves errors comparable to HLConcELM, and TP-ResNet significantly outperforms HLConcELM in terms of accuracy. Additionally, the training time of the four models remains nearly identical, suggesting that TP-ResNet provides the most effective approximation of Equation \eqref{eq:2D_function} without incurring extra computational cost.

Table \ref{tab:error_2d} presents detailed numerical results. As expected, training time increases with the number of basis functions, though all models complete training within $2$ seconds. For small $M$, all networks yield relatively large errors, but when $M \ge 1,600$, they all achieve significantly lower errors. The minimum $L_{\infty}$ and $L_{2}$ errors obtained by each model are $1.72 \times 10^{-8}$ and $4.82 \times 10^{-8}$ (HLConcELM), $1.87 \times 10^{-7}$ and $1.61 \times 10^{-6}$ (TP-ELM), $2.03 \times 10^{-8}$ and $1.02 \times 10^{-7}$ (TP-MLP), and $6.32 \times 10^{-11}$ and $3.44 \times 10^{-10}$ (TP-ResNet), respectively. These results confirm that while all four models perform well, TP-ResNet achieves the highest accuracy, making it the most effective architecture for this function approximation task.

\begin{figure}[htp]
    \centering
    \begin{minipage}{0.9\linewidth}
    \includegraphics[width=0.8\textwidth]{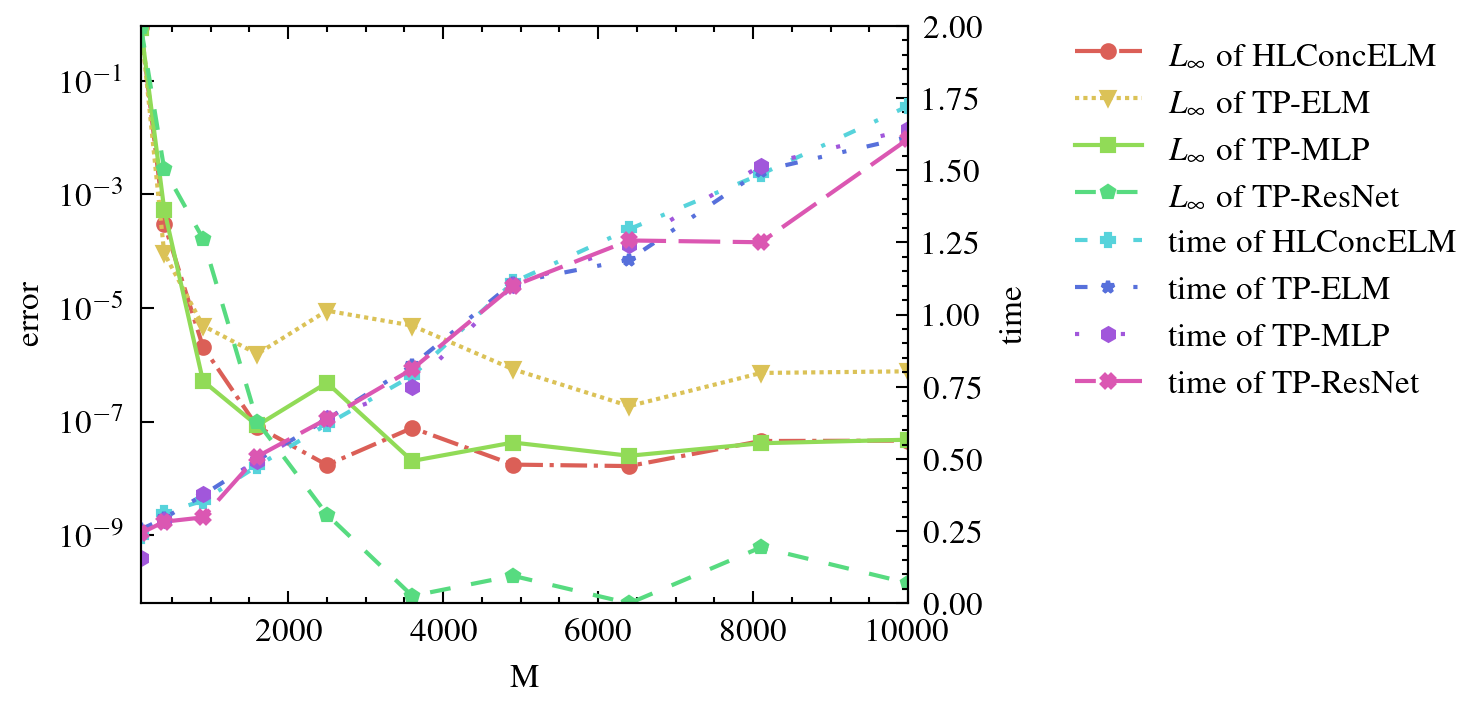}
    \end{minipage}
    \caption{Function approximation: The $L_{\infty}$ errors and the training time of neural networks.}
    \label{fig:error_2d}
\end{figure}

\begin{table}[htp]
    \begin{center}
        \caption{Function approximation: Performance comparison of the TPNet and HLConcELM \cite{ni2023numerical} in approximating the two-dimensional function defined by Equation \eqref{eq:2D_function}.  The $L_{\infty}$ errors, $L_{2}$ errors and training time for each model configuration are presented.}
        \setlength\tabcolsep{1pt}
        \tiny{
        \begin{tabular}{ccccccccccccc}
            \hline\noalign{\smallskip}
            \multirow{2}{*}{$M$} & \multicolumn{3}{c}{HLConcELM \cite{ni2023numerical}} & \multicolumn{3}{c}{TP-ELM} & \multicolumn{3}{c}{TP-MLP} & \multicolumn{2}{c}{TP-ResNet}  \\
            & $L_{\infty}$ & $L_{2}$ & time (s) & $L_{\infty}$ & $L_{2}$ & time (s) & $L_{\infty}$ & $L_{2}$ & time (s) & $L_{\infty}$ & $L_{2}$  & time (s)     \\
            \hline
            100     & 8.77E-01 & 1.94E+01 & 0.2561  & 9.22E-01 & 1.94E+01 & 0.1562 & 8.60E-01 & 2.36E+01 & 0.2331 & 8.77E-01 & 1.59E+01 & 0.2433  \\
            400     & 3.05E-04 & 1.82E-03 & 0.2947  & 9.16E-05 & 9.95E-04 & 0.2883 & 5.37E-04 & 6.39E-03 & 0.3134 & 2.75E-03 & 4.64E-02 & 0.2827  \\
            900     & 2.07E-06 & 8.09E-06 & 0.3751  & 4.86E-06 & 1.48E-05 & 0.3782 & 5.18E-07 & 2.10E-06 & 0.3554 & 1.66E-04 & 1.18E-03 & 0.2977  \\
            1600    & 8.16E-08 & 3.99E-07 & 0.4952  & 1.55E-06 & 2.25E-05 & 0.4921 & 8.52E-08 & 4.53E-07 & 0.4756 & 9.95E-08 & 8.47E-07 & 0.5084  \\
            2500    & 1.72E-08 & 1.07E-07 & 0.6375  & 8.94E-06 & 1.49E-05 & 0.6423 & 4.82E-07 & 1.09E-06 & 0.6198 & 2.31E-09 & 6.24E-09 & 0.6388  \\
            3600    & 7.74E-08 & 1.77E-07 & 0.8241  & 4.88E-06 & 8.86E-06 & 0.7495 & 2.03E-08 & 1.74E-07 & 0.7883 & 8.55E-11 & 6.42E-10 & 0.8113  \\
            4900    & 1.75E-08 & 4.82E-08 & 1.1107  & 8.34E-07 & 2.62E-06 & 1.1015 & 4.28E-08 & 1.35E-07 & 1.1116 & 1.93E-10 & 9.92E-10 & 1.0989  \\
            6400    & 1.65E-08 & 7.00E-08 & 1.1904  & 1.87E-07 & 1.86E-06 & 1.2404 & 2.50E-08 & 1.97E-07 & 1.2959 & 6.32E-11 & 3.44E-10 & 1.2567  \\
            8100    & 4.55E-08 & 8.56E-08 & 1.4978  & 7.17E-07 & 1.82E-06 & 1.5162 & 4.15E-08 & 1.02E-07 & 1.4862 & 6.11E-10 & 2.43E-09 & 1.2505  \\
            10000   & 4.62E-08 & 1.42E-07 & 1.6152  & 7.66E-07 & 1.61E-06 & 1.6415 & 4.80E-08 & 1.16E-07 & 1.7223 & 1.43E-10 & 7.01E-10 & 1.6082  \\
            \hline
        \end{tabular}
        }
        \label{tab:error_2d}
    \end{center}
\end{table}

\subsection{Helmholtz Equation}
We now consider the Helmholtz equation in two dimensions:
\begin{equation}
    \label{eq:Helmholtz_equation}
    \begin{array}{r@{}l}
        \left\{
        \begin{aligned}
            \Delta u + k^2 u & = q, &  & \mbox{in} \enspace \Omega,          \\
            u             & = h,        &  & \mbox{on} \enspace \partial \Omega,
        \end{aligned}
        \right.
    \end{array}
\end{equation}
where $\Omega = (-1, 1)^2$ represents the spatial domain, $\partial \Omega$ denotes its boundary. The exact solution is defined by:
\begin{equation}
    \label{eq:Helmholtz_solution}
    u(x, y) = \sin(a_1 \pi x) \sin(a_2 \pi y),
\end{equation}
with the source term given by:
\begin{equation}
    \label{eq:Helmholtz_source_term}
    q(x, y)=(k^2-(a_1 \pi)^2 - (a_2 \pi)^2) \sin(a_1 \pi x) \sin(a_2 \pi y),
\end{equation}
where the parameters are $a_1 = 1$, $a_2 = 4$, and $k = 1$.

For training, we use $N_x \times N_y = 101 \times 101$ uniform collocation points. The hidden layers in HLConcELM and the subnetworks in TPNet are initialized using the Kaiming initialization method. 

The $L_{\infty}$ error and training time for HLConcELM, TP-ELM, TP-MLP and TP-ResNet are presented in Figure \ref{fig:error_helmholtz}. It is observed that as the number of basis functions increases, all models show decreasing errors, which level off in the later stages.
Additionally, HLConcELM, TP-ELM, and TP-MLP fit the solution of Equation \eqref{eq:Helmholtz_equation} well as the number of basis functions increases. TP-ResNet, while initially exhibiting relatively larger errors, outperforms the other models when $M \ge 3,600$. TP-ELM’s errors are slightly larger than HLConcELM’s, but TP-ResNet achieves the best performance overall. The training time for TP-ELM, TP-MLP, and TP-ResNet is very close and relatively small, while the training time for HLConcELM is significantly larger than the other models.

HLConcELM, TP-ELM and TP-MLP fit the solution of Equation \eqref{eq:Helmholtz_equation} well as the number of basis functions increases, while the errors of TP-ResNet are relatively larger compared to them. However, when the number of basis functions $M \ge 3,600$, TP-ResNet achieves smaller errors than others. Similarly, when there are enough basis functions, we observe that the errors of TP-ELM are relatively larger compared to HLConcELM, the errors of TP-MLP are close to those of HLConcELM, and TP-ResNet performs the best. The training time of TP-ELM, TP-MLP and TP-ResNet is very close and relatively small, while the training time of HLConcELM is significantly larger than others.

Table \ref{tab:error_helmholtz} compares the approximation errors and training time for the different models. It is evident that for smaller $M$, none of the neural networks can approximate the Helmholtz equation well. TP-ResNet achieves the smallest error, although HLConcELM performs better than TP-ELM and TP-MLP in terms of error. However, HLConcELM’s training time is much longer --- especially for large $M$ --- with a time of $3,755.77$ seconds for $M=10,000$, while the other models finish in around $5$ seconds.

These results highlight the trade-off between accuracy and computational efficiency. TP-ResNet provides the most accurate solution, while TP-ELM and TP-MLP offer a good balance between accuracy and speed. HLConcELM, though accurate, is less efficient and may be suitable only when computational resources are less constrained.

\begin{figure}[htbp]
    \begin{minipage}{0.9\linewidth}
        \centering
        \includegraphics[width=0.8\textwidth]{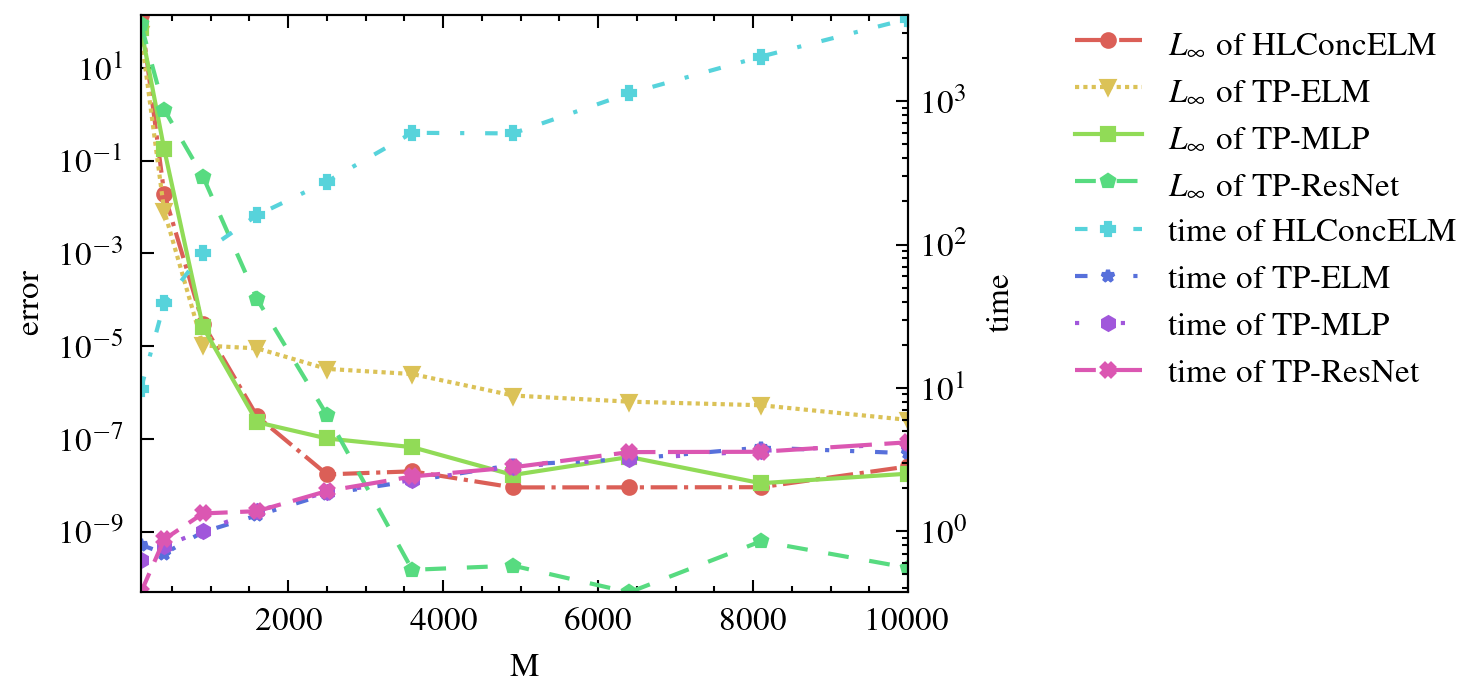}
    \end{minipage}
    \caption{Helmholtz equation: The $L_{\infty}$ errors and training time for the TPNet and HLConcELM \cite{ni2023numerical} when solving the Helmholtz Equation~\eqref{eq:Helmholtz_equation}.}
    \label{fig:error_helmholtz}
\end{figure}

\begin{table}[htp]
    \begin{center}
        \caption{Helmholtz equation: Performance comparison of approximation errors and training time for the TPNet and HLConcELM \cite{ni2023numerical} when solving the Helmholtz Equation~\eqref{eq:Helmholtz_equation}.}
        \setlength\tabcolsep{1pt}
        \tiny{
        \begin{tabular}{ccccccccccccc}
            \hline\noalign{\smallskip}
            \multirow{2}{*}{$M$} & \multicolumn{3}{c}{HLConcELM \cite{ni2023numerical}} & \multicolumn{3}{c}{TP-ELM} & \multicolumn{3}{c}{TP-MLP} & \multicolumn{2}{c}{TP-ResNet}  \\
            & $L_{\infty}$ & $L_{2}$ & time (s) & $L_{\infty}$ & $L_{2}$ & time (s) & $L_{\infty}$ & $L_{2}$ & time (s) & $L_{\infty}$ & $L_{2}$  & time (s)     \\
            \hline
            100    & 1.39E+02 & 1.80E+03  & 9.8950     & 5.34E+01 & 1.00E+03 & 0.8062 & 7.15E+01 & 1.89E+03 & 0.6258 & 8.78E+01 & 2.14E+03 & 0.3752  \\
            400    & 1.93E-02 & 2.29E-01  & 39.0217    & 8.21E-03 & 8.07E-02 & 0.6993 & 1.75E-01 & 1.98E+00 & 0.7860 & 1.26E+00 & 2.16E+01 & 0.8791  \\
            900    & 3.06E-05 & 6.07E-04  & 87.7262    & 1.04E-05 & 5.37E-04 & 0.9889 & 2.59E-05 & 2.29E-04 & 1.0127 & 4.54E-02 & 5.58E-01 & 1.3317  \\
            1600   & 3.05E-07 & 5.66E-06  & 160.4495   & 8.97E-06 & 1.61E-04 & 1.2908 & 2.29E-07 & 5.52E-06 & 1.3924 & 1.02E-04 & 8.06E-04 & 1.3821  \\
            2500   & 1.72E-08 & 3.02E-07  & 274.9721   & 3.24E-06 & 1.30E-04 & 1.8245 & 1.03E-07 & 2.90E-06 & 1.8678 & 3.27E-07 & 1.37E-06 & 1.9094  \\
            3600   & 2.01E-08 & 6.47E-07  & 601.4817   & 2.51E-06 & 7.52E-05 & 2.2495 & 6.66E-08 & 2.68E-06 & 2.2727 & 1.51E-10 & 2.93E-09 & 2.4121  \\
            4900   & 9.02E-09 & 3.24E-07  & 595.8922   & 8.64E-07 & 1.43E-05 & 2.9007 & 1.67E-08 & 3.64E-07 & 2.8221 & 1.86E-10 & 3.00E-09 & 2.7936  \\
            6400   & 9.05E-09 & 2.85E-07  & 1134.7612  & 6.37E-07 & 1.77E-05 & 3.1962 & 4.09E-08 & 1.42E-06 & 3.1929 & 4.93E-11 & 1.48E-09 & 3.5668  \\
            8100   & 9.12E-09 & 2.27E-07  & 2045.1812  & 5.34E-07 & 2.13E-05 & 3.8363 & 1.11E-08 & 5.90E-07 & 3.6376 & 6.22E-10 & 1.50E-08 & 3.5862  \\
            10000  & 2.56E-08 & 6.05E-07  & 3755.7702  & 2.58E-07 & 7.30E-06 & 3.5041 & 1.79E-08 & 4.46E-07 & 4.0776 & 1.66E-10 & 4.89E-09 & 4.1783  \\
            \hline
        \end{tabular}
        }
        \label{tab:error_helmholtz}
    \end{center}
\end{table}

In Table \ref{tab:error_Helmholtz2D_comparison}, we compare the performance of FEM, HLConcELM \cite{ni2023numerical}, RFM \cite{chen2022bridging}, RNN-PG \cite{shang2023randomized}, and TPNet in solving the Helmholtz equation \eqref{eq:Helmholtz_equation}, reporting key metrics including degrees of freedom (DoF), $L_{\infty}$ and $L_2$ errors, training time and number of parameters.

The FEM employs a fine mesh with grid size $h=0.0005$, yielding in a very large system with $741,209$ DoF.  Even with this high resolution, its accuracy is limited to $L_{\infty}$ and $L_2$ errors of $1.38 \times 10^{-6}$ and $9.42 \times 10^{-7}$, respectively. For the neural-network-based methods, the DoF corresponds to the number of basis functions, uniformly set to 10,000 in all comparisons. However, their parameter counts differ substantially due to architectural differences. HLConcELM has the largest parameter count due to its complex hierarchical and concatenated structure. In contrast, RFM and RNN-PG employ simple single-hidden-layer networks and consequently exhibit much lower parameter counts.  Our proposed TP-ELM also utilizes single-hidden-layer subnetworks but further reduces the parameter number by decomposing the problem into two smaller networks. Although TP-MLP and TP-ResNet adopt deeper or more structured architectures, their total parameter counts remain on the order of $10^4$, which is comparable to those of RFM and RNN‑PG.

In terms of training time, methods relying on automatic differentiation typically require hundreds to thousands of seconds. In contrast, TPNet variants complete training within seconds. Although FEM is a direct solver, it incurs high computational costs due to the assembly and solution of large sparse systems over thousands of elements. In terms of accuracy, TP-ResNet achieves the lowest of $L_{\infty}$ error,  with an $L_2$ error of $4.89 \times 10^{-9}$. Overall, Table \ref{tab:error_Helmholtz2D_comparison} demonstrates a compelling trade-off. TPNet achieves a superior or competitive accuracy with significantly lower computational cost and memory usage compared to classical FEM and other neural solvers.

\begin{table}[htp]
    \begin{center}
        \caption{Helmholtz equation: Performance comparison of approximation errors and training time for the TPNet, FEM, HLConcELM \cite{ni2023numerical}, RFM\cite{chen2022bridging} and RNN-PG \cite{shang2023randomized} when solving the Helmholtz Equation~\eqref{eq:Helmholtz_equation}.}
        \setlength\tabcolsep{2pt}
        \small{
        \begin{tabular}{llllll}
            \hline\noalign{\smallskip}
            Method & DoF & $L_{\infty} $ & $L_{2}$ & time (s) & params  \\
            \hline
            FEM & 741209 & 1.38E-06 & 9.42E-07 & 27904.0132 &  7.41E+05 \\
            HLConcELM \cite{ni2023numerical} &  10000  & 2.56E-08 & 6.05E-07 & 3755.7702 & 2.50E+07 \\
            RFM \cite{chen2022bridging} & 10000  & 4.66E-09 & 5.23E-08 & 588.5597 & 3.00E+04 \\
            RNN-PG \cite{shang2023randomized} & 10000  & 3.10E-06 & 2.94E-06 & 1242.2191 & 4.00E+04 \\
            TP-ELM & 10000 & 2.58E-07 & 7.30E-06 & 3.5041    & 6.00E+02 \\
            TP-MLP &  10000 & 1.79E-08 & 4.46E-07 & 4.0776   & 6.12E+04 \\
            TP-ResNet & 10000 & 1.66E-10 & 4.89E-09 & 4.1783 & 6.18E+04 \\
            \hline
        \end{tabular}
        }
        \label{tab:error_Helmholtz2D_comparison}
    \end{center}
\end{table}

\subsection{Heat Equation}
We consider the following heat equation within the spatial-temporal domain $\Omega \times T = (0, 1)^2 \times (0, 1]$:
\begin{equation}
    \label{eq:Heat_equation}
    \begin{array}{r@{}l}
        \left\{
        \begin{aligned}
            u_t(x,y,t) -\Delta u(x,y,t) & = f(x,y,t), &  & (x,y,t) \in \Omega \times (0, 1],          \\
            u(x,y,t)             & = g(x,y,t),        &  & (x,y,t) \in \partial \Omega \times [0, 1], \\
            u(x, y, 0)             & = h(x, y),        &  & (x, y) \in \Omega. \\
        \end{aligned}
        \right.
    \end{array}
\end{equation}
We choose the suitable forcing function $f(x,y,t)$ so that the exact solution is given by $u(x, y, t)=2 e^{-t} \sin(\frac{\pi}{2}x) \sin(\frac{\pi}{2}y)$.

In our approach, the temporal variable is treated equivalently to spatial variables, enabling a unified solution process. During the training stage, we select a uniform grid of $N_x \times N_y \times N_t = 51 \times 51 \times 51$ collocation points. The input layers of TPNet and HLConcELM are designed with three neurons, corresponding to the three-dimensional input space, which includes two spatial dimensions and one temporal dimension. Unlike the previous sections where TPNet and HLConcELM were initialized using the Kaiming method, in this example, we employ the Xavier initialization method due to its greater stability in high-dimensional problems.

Figure \ref{fig:error_heat} presents the $L_{\infty}$ errors and training time for TPNet and HLConcELM in solving the heat equation \eqref{eq:Heat_equation}. The errors for TP-ELM, TP-MLP, and TP-ResNet decrease rapidly before stabilizing, whereas the errors for HLConcELM initially decrease but then increase as the number of basis functions grows. TP-ELM demonstrates the fastest error reduction, followed closely by TP-MLP and HLConcELM, with TP-ResNet exhibiting the slowest reduction. However, when the number of basis functions is sufficiently large, TP-ResNet achieves smaller errors than HLConcELM. Additionally, as the number of basis functions increases, the training time rises for all four networks, with TP-ELM requiring the least time, followed by TP-MLP, TP-ResNet, and HLConcELM. Notably, HLConcELM demands more time to solve the heat equation \eqref{eq:Heat_equation} compared to the other methods.

Table \ref{tab:error_heat} further demonstrates that HLConcELM not only fails to achieve lower errors but also requires significantly more training time compared to the other networks. TP-MLP achieves the best results, with the smallest error. Although TP-ResNet's optimal error is comparable to that of TP-MLP, its corresponding training time is approximately $27$ seconds longer. While TP-ELM exhibits a slightly higher error than TP-MLP, it still attains a very low error with a training time of only $19$ seconds. Overall, TPNet effectively solves the heat equation while maintaining high computational efficiency.

\begin{figure}[htbp]
    \begin{minipage}{0.9\linewidth}
        \centering
        \includegraphics[width=0.8\textwidth]{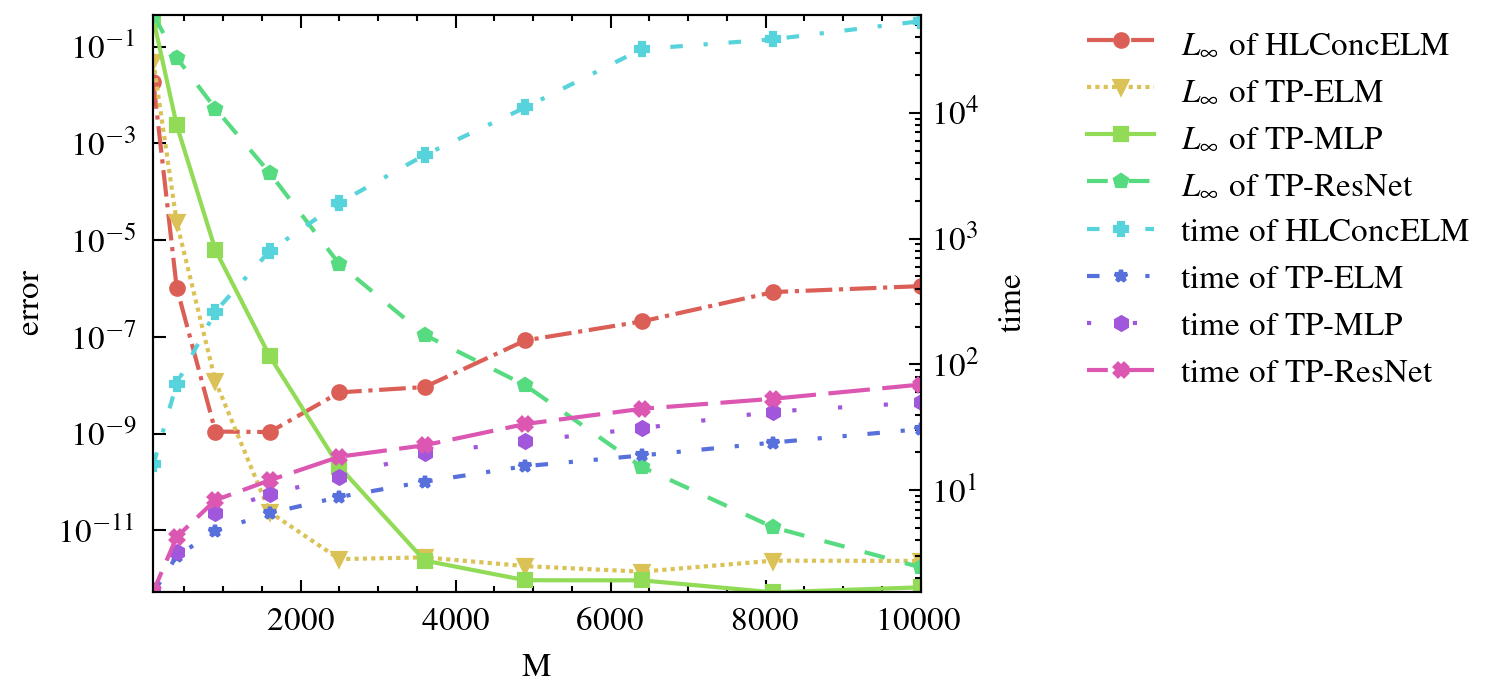}
    \end{minipage}
    \caption{Heat equation: The $L_{\infty}$ errors and training time for the TPNet and HLConcELM when solving the heat equation~\eqref{eq:Heat_equation}.}
    \label{fig:error_heat}
\end{figure}

\begin{table}[htp]
    \begin{center}
        \caption{Heat equation: Performance comparison of the TPNet and HLConcELM \cite{ni2023numerical} in approximating the heat equation~\eqref{eq:Heat_equation}.}
        \setlength\tabcolsep{1pt}
        \tiny{
        \begin{tabular}{ccccccccccccc}
            \hline\noalign{\smallskip}
            \multirow{2}{*}{$M$} & \multicolumn{3}{c}{HLConcELM \cite{ni2023numerical}} & \multicolumn{3}{c}{TP-ELM} & \multicolumn{3}{c}{TP-MLP} & \multicolumn{2}{c}{TP-ResNet}  \\
            & $L_{\infty}$ & $L_{2}$ & time (s) & $L_{\infty}$ & $L_{2}$ & time (s) & $L_{\infty}$ & $L_{2}$ & time (s) & $L_{\infty}$ & $L_{2}$  & time (s)     \\
            \hline
            100    & 1.83E-02 & 2.98E-01 & 16.0370     & 4.86E-02 & 5.43E-01 & 1.5451    & 3.63E-01 & 5.29E+00 & 1.6130  & 4.54E-01 & 7.39E+00 & 1.5765   \\
            400    & 1.05E-06 & 8.69E-06 & 70.2995     & 2.35E-05 & 2.28E-04 & 2.9935    & 2.40E-03 & 3.67E-02 & 3.2216  & 5.89E-02 & 1.16E+00 & 4.2597   \\
            900    & 1.11E-09 & 1.21E-08 & 263.3418    & 1.22E-08 & 1.03E-07 & 4.7488    & 6.25E-06 & 6.11E-05 & 6.6332  & 5.23E-03 & 6.90E-02 & 8.3511   \\
            1600   & 1.08E-09 & 1.13E-08 & 796.6640    & 2.40E-11 & 2.47E-10 & 6.5836    & 4.01E-08 & 2.99E-07 & 9.3437  & 2.45E-04 & 2.14E-03 & 12.0106  \\
            2500   & 7.14E-09 & 9.02E-08 & 1929.5530   & 2.56E-12 & 2.95E-11 & 8.8515    & 2.06E-10 & 1.19E-09 & 12.7047 & 3.18E-06 & 2.07E-05 & 18.4881  \\
            3600   & 9.20E-09 & 1.20E-07 & 4618.2781   & 2.78E-12 & 1.98E-11 & 11.7062   & 2.38E-12 & 2.31E-11 & 19.6833 & 1.10E-07 & 6.97E-07 & 22.7907  \\
            4900   & 8.48E-08 & 9.60E-07 & 11028.6781  & 1.83E-12 & 2.28E-11 & 15.5527   & 9.35E-13 & 1.03E-11 & 24.5404 & 9.97E-09 & 4.63E-08 & 33.6471  \\
            6400   & 2.09E-07 & 2.54E-06 & 32061.6871  & 1.41E-12 & 2.32E-11 & 18.9126   & 9.31E-13 & 1.46E-11 & 31.1021 & 2.06E-10 & 1.08E-09 & 44.3981  \\
            8100   & 8.46E-07 & 1.20E-05 & 38383.0819  & 2.37E-12 & 2.41E-11 & 23.9610   & 5.26E-13 & 7.12E-12 & 42.2084 & 1.16E-11 & 9.15E-11 & 53.3838  \\
            10000  & 1.12E-06 & 1.40E-05 & 53451.3578  & 2.35E-12 & 2.87E-11 & 30.6428   & 6.68E-13 & 1.22E-11 & 49.8507 & 1.73E-12 & 1.08E-11 & 69.2094  \\
            \hline
        \end{tabular}
        }
        \label{tab:error_heat}
    \end{center}
\end{table}

\subsection{Wave Equation}
We consider the following wave equation within $\Omega  \times T = (0, 1)^2 \times  (0, 1]$
%, given by the system of PDEs:
\begin{equation}
    \label{eq:Wave_equation}
    \begin{array}{r@{}l}
        \left\{
        \begin{aligned}
            \frac{\partial ^2 u}{\partial t^2} -\Delta u(x,y,t) & = f(x,y,t), &  & (x,y,t) \in \Omega  \times (0, 1],          \\
            u(x,y,t)             & = g(x,y,t),        &  & (x,y,t) \in \partial \Omega \times [0, 1], \\
            u(x, y, 0)             & = h(x, y),        &  & (x, y) \in \Omega, \\
            \frac{\partial u}{\partial t} (x, y, 0) & = w(x, y), & & (x, y) \in \Omega, 
        \end{aligned}
        \right.
    \end{array}
\end{equation}
where the exact solution is $u(x, y, t)=\sin(\frac{\pi}{2}x) \sin(\frac{\pi}{2}y) \sin(\frac{\pi}{2}t)$. The boundary condition $g(x, y, t)$, the initial conditions $h(x, y, t)$ and $w(x, y)$ and the right-hand source term $f(x, y, t)$ are given accordingly.

We establish a grid of $N_x \times N_y \times N_t = 51 \times 51 \times 51$ uniform collocation points. The TPNet and HLConcELM are initialized using the Xavier initialization method. The $L_{\infty}$ errors and training time for TPNet and HLConcELM are shown in Figure \ref{fig:error_wave}. It is evident that the errors of TP-ELM, TP-MLP, and TP-ResNet decrease with an increasing number of basis functions. However, the errors of HLConcELM decrease initially and increase after $M>1,600$. For $M \le 1,600$, HLConcELM has the fastest error reduction, followed by TP-ELM, TP-MLP, and TP-ResNet. Although TP-ELM shows a rapid decrease in error, TP-MLP achieves lower errors when the number of basis functions is sufficiently large, and TP-ResNet also attains errors close to those of TP-ELM when $M=10,000$. It is evident that HLConcELM does not achieve lower errors compared to other neural networks. Furthermore, the training time of HLConcELM is markedly higher than that of the other neural networks. However, the training time for TP-ELM, TP-MLP, and TP-ResNet is relatively close, with TP-ELM requiring the least training time, followed by TP-MLP and TP-ResNet.

Table \ref{tab:error_wave} further demonstrates that HLConcELM has the longest training time and also achieves the largest error. Although the error of TP-ELM is an order of magnitude higher than that of TP-MLP, it only requires about $6$ seconds to complete training, while TP-MLP and TP-ResNet require $28$ seconds and $37$ seconds, respectively. Additionally, for $M = 10,000$, HLConcELM requires a significant amount of training time ($136,756.8406$ seconds), whereas the other neural networks complete training within $38$ seconds. In summary, TPNet demonstrates superior performance in both accuracy and efficiency when compared to HLConcELM. TP-ELM, in particular, stands out for its rapid training capability, making it a suitable choice for scenarios where computational efficiency is crucial.

\begin{figure}[htbp]
    \begin{minipage}{0.9\linewidth}
        \centering
        \includegraphics[width=0.8\textwidth]{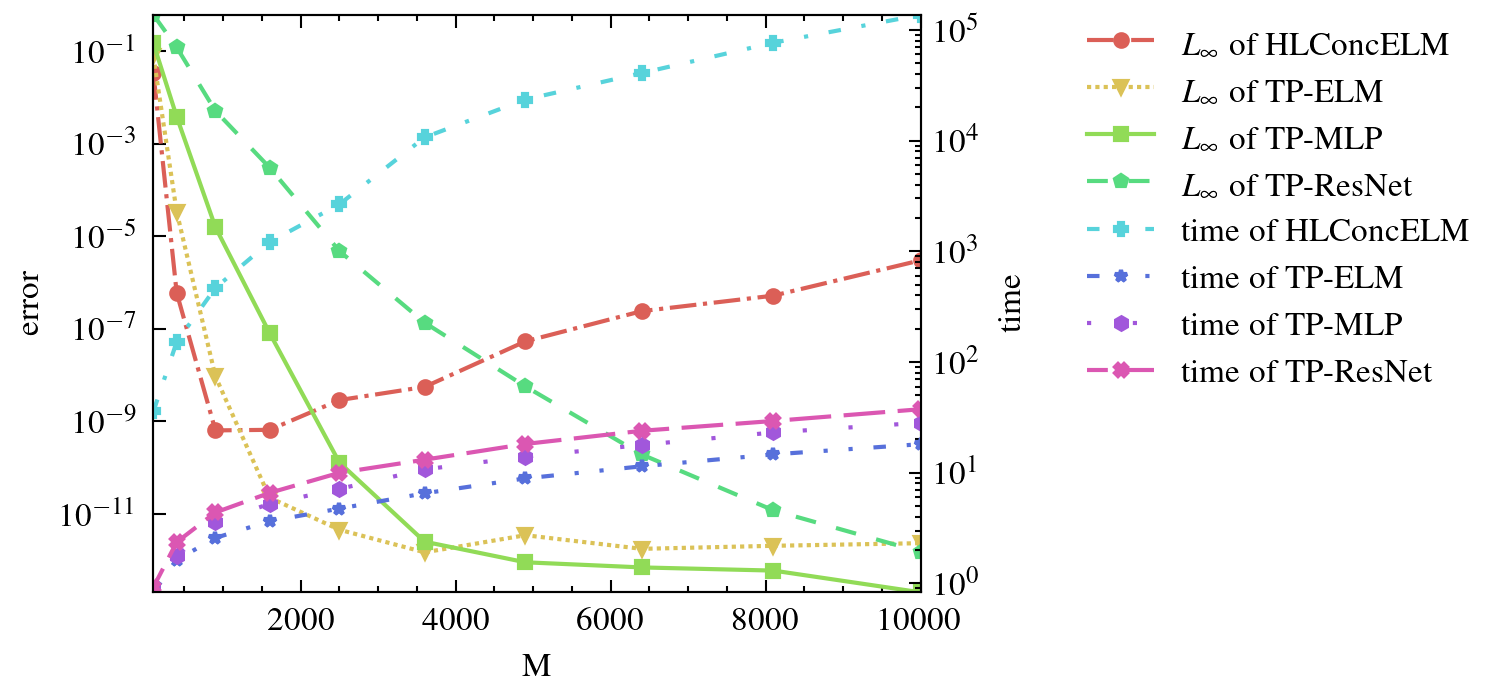}
    \end{minipage}
    \caption{Wave equation: The $L_\infty$ errors and training time  for the TPNet and HLConcELM \cite{ni2023numerical} when solving the wave equation~\eqref{eq:Wave_equation}.}    
    \label{fig:error_wave}
\end{figure}

\begin{table}[htp]
    \begin{center}
        \caption{Wave equation: Performance comparison of the TPNet and HLConcELM \cite{ni2023numerical} in approximating the wave equation~\eqref{eq:Wave_equation}.}
        \setlength\tabcolsep{1pt}
        \tiny{
        \begin{tabular}{ccccccccccccc}
            \hline\noalign{\smallskip}
            \multirow{2}{*}{$M$} & \multicolumn{3}{c}{HLConcELM \cite{ni2023numerical}} & \multicolumn{3}{c}{TP-ELM} & \multicolumn{3}{c}{TP-MLP} & \multicolumn{2}{c}{TP-ResNet}  \\
            & $L_{\infty}$ & $L_{2}$ & time (s) & $L_{\infty}$ & $L_{2}$ & time (s) & $L_{\infty}$ & $L_{2}$ & time (s) & $L_{\infty}$ & $L_{2}$  & time (s)     \\
            \hline
            100   & 3.32E-02 & 2.82E-01 & 35.7711      & 7.62E-02 & 7.96E-01 & 0.8297   & 1.46E-01 & 2.17E+00 & 0.8985  & 6.04E-01 & 1.37E+01 & 0.9239   \\
            400   & 5.99E-07 & 5.20E-06 & 151.5896     & 3.38E-05 & 2.19E-04 & 1.6280   & 3.82E-03 & 5.13E-02 & 1.7670  & 1.20E-01 & 1.55E+00 & 2.3565   \\
            900   & 6.37E-10 & 7.76E-09 & 467.0246     & 9.34E-09 & 5.41E-08 & 2.5748   & 1.56E-05 & 8.86E-05 & 3.5592  & 5.09E-03 & 5.30E-02 & 4.3725   \\
            1600  & 6.55E-10 & 9.05E-09 & 1211.2063    & 2.24E-11 & 1.61E-10 & 3.6484   & 8.03E-08 & 4.88E-07 & 5.2345  & 3.02E-04 & 2.41E-03 & 6.5723   \\
            2500  & 2.88E-09 & 3.83E-08 & 2658.4405    & 4.55E-12 & 3.54E-11 & 4.7364   & 1.29E-10 & 9.46E-10 & 7.1391  & 4.77E-06 & 2.39E-05 & 9.9762   \\
            3600  & 5.55E-09 & 6.44E-08 & 10696.9190   & 1.48E-12 & 1.35E-11 & 6.5149   & 2.51E-12 & 2.35E-11 & 10.7047 & 1.35E-07 & 7.94E-07 & 13.1009  \\
            4900  & 5.34E-08 & 6.65E-07 & 23501.8955   & 3.47E-12 & 1.52E-11 & 8.8781   & 9.01E-13 & 7.88E-12 & 13.9505 & 5.67E-09 & 4.88E-08 & 18.1687  \\
            6400  & 2.40E-07 & 3.42E-06 & 40756.9741   & 1.76E-12 & 1.31E-11 & 11.4151  & 6.98E-13 & 5.52E-12 & 17.8052 & 1.93E-10 & 1.59E-09 & 23.8998  \\
            8100  & 5.19E-07 & 8.38E-06 & 76619.7599   & 2.05E-12 & 2.74E-11 & 14.6303  & 5.96E-13 & 3.61E-12 & 23.0901 & 1.20E-11 & 1.10E-10 & 29.3154  \\
            10000 & 3.01E-06 & 5.48E-05 & 136756.8406  & 2.33E-12 & 2.27E-11 & 18.0904  & 2.02E-13 & 2.07E-12 & 28.2356 & 1.51E-12 & 1.05E-11 & 37.4434  \\
            \hline
        \end{tabular}
        }
        \label{tab:error_wave}
    \end{center}
\end{table}

\subsection{Nonlinear Burger's Equation}
We now consider the nonlinear Burger's equation in the domain $\Omega \times T = (0,1)^2 \times (0,1]$, formulated as follows:
\begin{equation}
    \label{eq:nonlinear_Burgers_equation}
    \begin{array}{r@{}l}
        \left\{
        \begin{aligned}
            u_t + u(u_x + u_y) - \zeta \Delta u & = 0, &  & (x, y, t) \in \Omega  \times (0, 1],          \\
            u(x, y, t)             & = g(x, y, t), & & (x, y, t) \in \partial \Omega \times [0, 1],     \\
            u(x, y, 0)             & = h(x, y),   & & (x, y) \in \Omega. \\
        \end{aligned}
        \right.
    \end{array}
\end{equation}
Here, the exact solution is given by $u(x, y, t) = \frac{1}{1+e^{(x+y-t)/(2 \zeta)}}$, where the constant $\zeta$ is set to $1$.

In the training stage, a uniform grid of $N_x \times N_y \times N_t = 51 \times 51 \times 51$ collocation points is used. Algorithm~\ref{algo:tp_nonlinear_pde} is employed to iteratively solve this nonlinear problem. The initial weight vector $\boldsymbol{w}_0$ is initialized using the Xavier initialization method. The iterative process is controlled by a maximum of nonlinear iterations $k_{\max} = 100$ and a stopping criterion $\epsilon = 10^{-16}$ to ensure convergence.

The results are presented in Figure~\ref{fig:error_burgers} and Table~\ref{tab:error_burgers}. The errors of TP-ELM, TP-MLP, and TP-ResNet decrease with an increasing number of basis functions. However, HLConcELM exhibits an increase in error when the number of basis functions exceeds $M > 1,600$. For $M \leq 1,600$, HLConcELM demonstrates the fastest error reduction rate, followed by TP-ELM, TP-MLP, and TP-ResNet. When a sufficient number of basis functions is used, TP-ELM and TP-MLP achieve lower errors than HLConcELM, while TP-ResNet attains errors close to those of HLConcELM. As shown in Figure~\ref{fig:error_burgers}, the training time of TP-ELM, TP-MLP, and TP-ResNet is nearly identical, whereas HLConcELM requires significantly more computation time.

These results highlight the trade-offs between accuracy, computational efficiency, and network architecture in solving the nonlinear Burgers' equation using TPNet. The $L_{\infty}$ errors, $L_2$ errors, and training time are summarized in Table~\ref{tab:error_burgers}. TPNet achieves higher precision solutions compared to HLConcELM. Notably, the training time for TPNet in this problem is longer than in other cases due to the necessity of using a full matrix in nonlinear iterations. The primary advantage of TPNet lies in its ability to construct a smaller neural network while maintaining a sufficient number of basis functions, thereby improving differentiation efficiency.

\begin{figure}[htbp]
    \begin{minipage}{0.9\linewidth}
        \centering
        \includegraphics[width=0.8\textwidth]{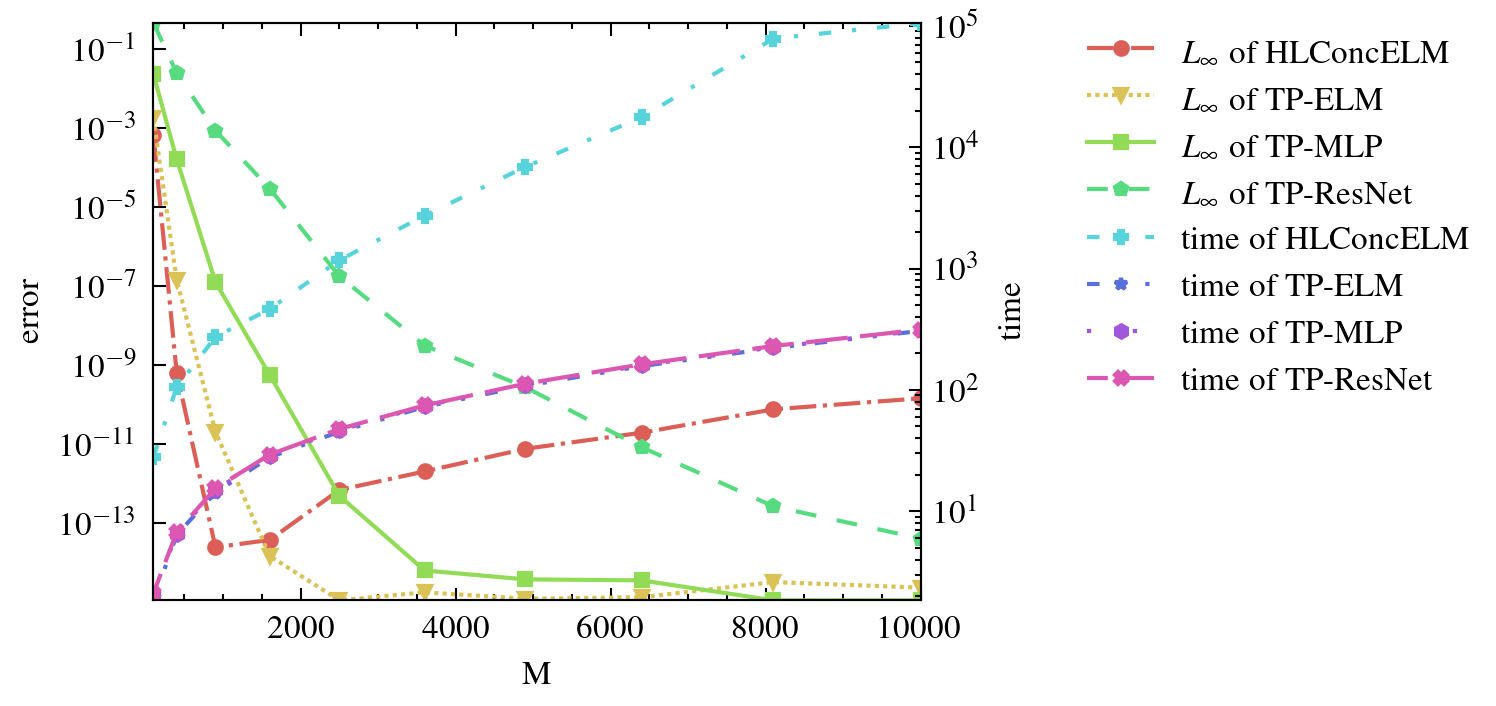}
    \end{minipage}
    \caption{Nonlinear Burger's equation: The $L_{\infty}$ errors and training time for the TPNet and HLConcELM \cite{ni2023numerical} in solving the nonlinear Burger's equation~\eqref{eq:nonlinear_Burgers_equation}.}    
    \label{fig:error_burgers}
\end{figure}

\begin{table}[htp]
    \begin{center}
        \caption{Nonlinear Burger's equation: Performance comparison of TPNet and HLConcELM \cite{ni2023numerical} for solving the nonlinear Burger's equation~\eqref{eq:nonlinear_Burgers_equation} with $\zeta = 1$.}
        \setlength\tabcolsep{1pt}
        \tiny{
        \begin{tabular}{ccccccccccccc}
            \hline\noalign{\smallskip}
            \multirow{2}{*}{$M$} & \multicolumn{3}{c}{HLConcELM \cite{ni2023numerical}} & \multicolumn{3}{c}{TP-ELM} & \multicolumn{3}{c}{TP-MLP} & \multicolumn{2}{c}{TP-ResNet}  \\
            & $L_{\infty}$ & $L_{2}$ & time (s) & $L_{\infty}$ & $L_{2}$ & time (s) & $L_{\infty}$ & $L_{2}$ & time (s) & $L_{\infty}$ & $L_{2}$  & time (s)     \\
            \hline
            100   & 6.66E-04 & 7.41E-03 & 28.0381      & 1.79E-03 & 2.36E-02 & 1.8437   & 2.34E-02 & 5.12E-01 & 2.0504   & 4.70E-01 & 1.74E+01 & 2.1043    \\
            400   & 6.50E-10 & 4.61E-09 & 105.2181     & 1.46E-07 & 1.14E-06 & 6.3930   & 1.66E-04 & 2.10E-03 & 6.5327   & 2.53E-02 & 3.93E-01 & 6.7319    \\
            900   & 2.45E-14 & 2.29E-13 & 273.0279     & 2.08E-11 & 1.76E-10 & 14.2564  & 1.28E-07 & 1.09E-06 & 14.9841  & 8.77E-04 & 1.27E-02 & 15.4513   \\
            1600  & 3.76E-14 & 5.16E-13 & 464.2237     & 1.48E-14 & 1.26E-13 & 27.8418  & 5.69E-10 & 3.78E-09 & 28.6185  & 2.99E-05 & 2.95E-04 & 29.2847   \\
            2500  & 6.95E-13 & 5.06E-12 & 1172.2580    & 1.11E-15 & 1.87E-14 & 45.3162  & 4.90E-13 & 3.47E-12 & 46.6289  & 1.78E-07 & 1.67E-06 & 47.6982   \\
            3600  & 2.04E-12 & 8.31E-12 & 2736.7118    & 1.78E-15 & 3.28E-14 & 72.4357  & 6.33E-15 & 7.24E-14 & 73.9668  & 3.15E-09 & 2.69E-08 & 74.5521   \\
            4900  & 7.65E-12 & 8.43E-11 & 6894.5059    & 1.22E-15 & 2.61E-14 & 108.2349 & 3.77E-15 & 5.15E-14 & 110.7099 & 2.76E-10 & 2.09E-09 & 112.5262  \\
            6400  & 1.96E-11 & 2.17E-10 & 17678.5187   & 1.33E-15 & 1.66E-14 & 157.1707 & 3.55E-15 & 7.94E-14 & 159.4094 & 8.38E-12 & 4.61E-11 & 162.4819  \\
            8100  & 7.64E-11 & 9.98E-10 & 78571.8761   & 3.22E-15 & 4.79E-14 & 221.3515 & 1.11E-15 & 1.74E-14 & 224.6569 & 2.67E-13 & 2.14E-12 & 230.0348  \\
            10000 & 1.45E-10 & 2.22E-09 & 105508.0333  & 2.33E-15 & 4.44E-14 & 306.6965 & 1.11E-15 & 2.45E-14 & 309.4374 & 4.09E-14 & 6.19E-13 & 314.3046  \\
            \hline
        \end{tabular}
        }
        \label{tab:error_burgers}
    \end{center}
\end{table}

\subsection{High-Dimensional Poisson Equation}
In this example, we demonstrate the capability of our proposed methods in handling high-dimensional problems by solving the following Poisson equation:
\begin{equation}
    \label{eq:HDPoisson_equation}
    \begin{array}{r@{}l}
        \left\{
        \begin{aligned}
            -\Delta u & = f(\boldsymbol{x}), &  & \mbox{in} \enspace \Omega,          \\
            u             & = h(\boldsymbol{x}),        &  & \mbox{on} \enspace \partial \Omega,
        \end{aligned}
        \right.
    \end{array}
\end{equation}
where $\Omega = (-1, 1)^d$. The exact solution to this equation is defined as
\begin{equation}
    \label{eq:HDPoisson_solution}
    u(\boldsymbol{x}) = \left( \frac{1}{d}\sum_{i=1}^d x_i \right)^2 + \sin \left( \frac{1}{d}\sum_{i=1}^d x_i \right).
\end{equation}
with the corresponding source term given by $f(\boldsymbol{x})=\frac{1}{d} \left( \sin \left( \frac{1}{d}\sum_{i=1}^d x_i \right) - 2 \right)$ and $h(\boldsymbol{x})=\left( \frac{1}{d}\sum_{i=1}^d x_i \right)^2 + \sin \left( \frac{1}{d}\sum_{i=1}^d x_i \right)$.

We randomly select $10,000$ collocation points in the domain $\Omega$ and $200d$ collocation points on the boundary $\partial \Omega$ using Latin Hypercube Sampling (LHS) \cite{stein1987large} to form the training dataset. This dataset encompasses points within the spatial domain $\Omega$ as well as points lying on its boundary $\partial \Omega$. Furthermore, the four neural networks in this example are initialized using the Xavier initialization method. To save computational time, we train the HLConcELM with $M=2,500$ basis functions and train the TP-ELM, TP-MLP, and TP-ResNet with $M=10,000$.

To verify the effectiveness of our proposed methods for solving high-dimensional problems, TPNet and HLConcELM are used to solve the Poisson equation for $d=5,7,10,15$. The results are shown in Figure \ref{fig:error_PoiHD} and Table \ref{tab:error_PoiHD}. As the dimension $d$ increases, the errors of the approximate solutions for all four neural networks also increase. The errors of TP-MLP and TP-ResNet are both smaller than those of HLConcELM, while the error of TP-ELM is greater than that of HLConcELM. Additionally, TP-ELM requires the least training time, whereas the training time of TP-ResNet is greater than that of TP-MLP. This is because, as the dimension increases, the impact of network complexity on the training time required to solve PDEs becomes more significant. Moreover, the training time of HLConcELM is much longer than that of the other neural networks.

The $L_{\infty}$ errors, $L_2$ errors, and training time of the four neural networks are listed in Table \ref{tab:error_PoiHD}. TP-ELM exhibits significantly larger $L_{\infty}$ and $L_2$ errors compared to HLConcELM. However, TP-MLP and TP-ResNet achieve smaller $L_{\infty}$ and $L_2$ errors than HLConcELM. Regarding training time, TP-ELM completes training in approximately $5$ seconds, whereas TP-MLP and TP-ResNet require around $100$ seconds. In contrast, HLConcELM requires several hundred seconds of training. Overall, the proposed TPNet demonstrates effectiveness in handling high-dimensional Poisson problems.

\begin{figure}[htbp]
    \begin{minipage}{0.9\linewidth}
        \centering
        \includegraphics[width=0.8\textwidth]{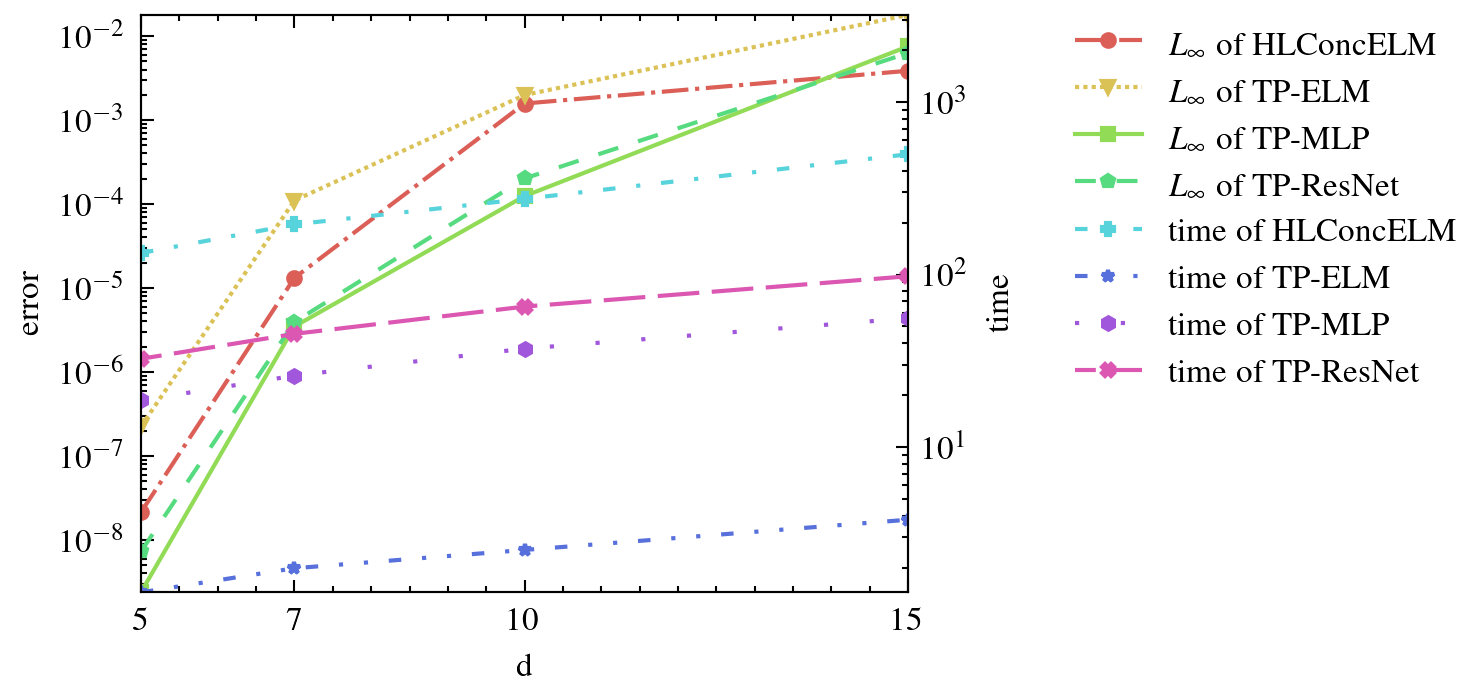}
    \end{minipage}
    \caption{High-dimensional Poisson equation: The $L_{\infty}$ errors and training time for the TPNet and HLConcELM \cite{ni2023numerical} when solving the equation~\eqref{eq:Helmholtz_equation}.}
    \label{fig:error_PoiHD}
\end{figure}

\begin{table}[htp]
    \begin{center}
        \caption{High-dimensional Poisson equation: Performance comparison of approximation errors and training time for the TPNet and HLConcELM \cite{ni2023numerical} when solving Equation~\eqref{eq:HDPoisson_equation}.}
        \setlength\tabcolsep{1pt}
        \tiny{
        \begin{tabular}{ccccccccccccc}
            \hline\noalign{\smallskip}
            \multirow{2}{*}{$M$} & \multicolumn{3}{c}{HLConcELM \cite{ni2023numerical}} & \multicolumn{3}{c}{TP-ELM} & \multicolumn{3}{c}{TP-MLP} & \multicolumn{2}{c}{TP-ResNet}  \\
            & $L_{\infty}$ & $L_{2}$ & time (s) & $L_{\infty}$ & $L_{2}$ & time (s) & $L_{\infty}$ & $L_{2}$ & time (s) & $L_{\infty}$ & $L_{2}$  & time (s)     \\
            \hline
            5   & 2.13E-08 & 7.83E-08 & 133.2592   & 2.24E-07 & 1.38E-06 & 1.4383  & 2.37E-09 & 1.32E-08 & 18.6470 & 7.28E-09 & 2.09E-08 & 32.4662  \\
            7   & 1.32E-05 & 1.02E-04 & 195.1425   & 1.10E-04 & 9.21E-04 & 1.9843  & 3.48E-06 & 3.38E-05 & 25.9330 & 3.93E-06 & 3.72E-05 & 45.2929  \\
            10  & 1.57E-03 & 1.41E-02 & 274.3841   & 1.99E-03 & 2.45E-02 & 2.5325  & 1.24E-04 & 1.45E-03 & 37.1885 & 2.02E-04 & 1.40E-03 & 65.1845  \\
            15  & 3.87E-03 & 6.53E-02 & 497.8459   & 1.79E-02 & 3.15E-01 & 3.7921  & 7.56E-03 & 1.17E-01 & 55.7701 & 6.30E-03 & 1.15E-01 & 97.7254  \\
            \hline
        \end{tabular}
        }
        \label{tab:error_PoiHD}
    \end{center}
\end{table}

\subsection{Block-Time Marching Test}
\label{sec:BTM}
In this subsection, we illustrate the efficiency of the block-time marching (BTM) strategy using a diffusion equation. The equation is given by:
\begin{equation}
    \label{eq:Diffusion_equation}
    \begin{array}{r@{}l}
        \left\{
        \begin{aligned}
            \frac{\partial u}{\partial t} - \nu \frac{\partial^2 u}{\partial x^2} & = f(x, t), &  & (x, t) \in (0,5) \times (0, t_f],          \\
            u(0, t)             & = g_1(t),        &  & t \in[0, t_f], \\
            u(5, t)             & = g_2(t),        &  & t \in[0, t_f], \\
            u(x, 0)             & = h(x),        &  & x \in (0,5). \\
        \end{aligned}
        \right.
    \end{array}
\end{equation}

Here, $f(x ,t)$ is the source term, $\nu = 0.01$ is a constant, and $g_1(t)$, $g_2(t)$, and $h(x)$ represent the boundary and initial conditions, respectively. The exact solution is:
\begin{equation}
    \label{eq:Diffusion_equation_solution}
    \begin{array}{r@{}l}
        \begin{aligned}
    u(x, t)=& \left[2 \cos\left(\pi x+\frac{\pi}{5}\right) + \frac{3}{2} \cos\left(2\pi x - \frac{3\pi}{5}\right)\right] \\
    & \times \left[2 \cos\left(\pi t+\frac{\pi}{5}\right) + \frac{3}{2} \cos\left(2\pi t - \frac{3\pi}{5}\right)\right].
        \end{aligned}
    \end{array}
\end{equation}

To solve the diffusion equation for long-time simulations with $t_f = 10$, we employ the BTM strategy. This method divides the simulation time into discrete blocks, each of size $dt = 1$, making computations more manageable. We set the number of time blocks to $D_t = 10$ and use the same neural network architecture for each block, ensuring a consistent learning approach across all stages. The number of basis functions is set to $M = 10,000$, and each hidden and output layer in the subnetworks of TPNet contains $100$ neurons, providing a rich representation of the solution space.

Figure \ref{fig:error_Diff_tf10} and Table \ref{tab:error_Diff_tf10} compare the performance of TPNet and HLConcELM in solving this problem, both with and without BTM. The error curves demonstrate that TP-ResNet consistently achieves the lowest errors, regardless of whether BTM is applied. TP-MLP shows marginally smaller errors than HLConcELM, whereas TP-ELM exhibits the highest errors.  The use of BTM significantly reduces errors, highlighting its efficacy in long-time simulations. Additionally, the training time curves reveal that TPNet trains substantially faster than HLConcELM. When BTM is utilized, the training time of HLConcELM increases sharply, while TPNet completes training in under $31$ seconds.

Figure \ref{fig:heatmap3_diffusion_10} presents the results of applying TP-ResNet using BTM to the diffusion equation. The heatmap displays the absolute error between the exact and approximate solutions, which is on the order of $10^{-8}$, indicating that TP-ResNet accurately captures the solution. The $L_{\infty}$ error, $L_2$ error, and training time for TP-ResNet are $1.92 \times 10^{-8}$, $1.61 \times 10^{-6}$, and $31.0119$ seconds, respectively. These results confirm that TP-ResNet efficiently provides high-precision solutions for long-time simulations of diffusion equations.

\begin{figure}[htbp]
    \begin{minipage}{0.9\linewidth}
        \centering
        \includegraphics[width=0.8\textwidth]{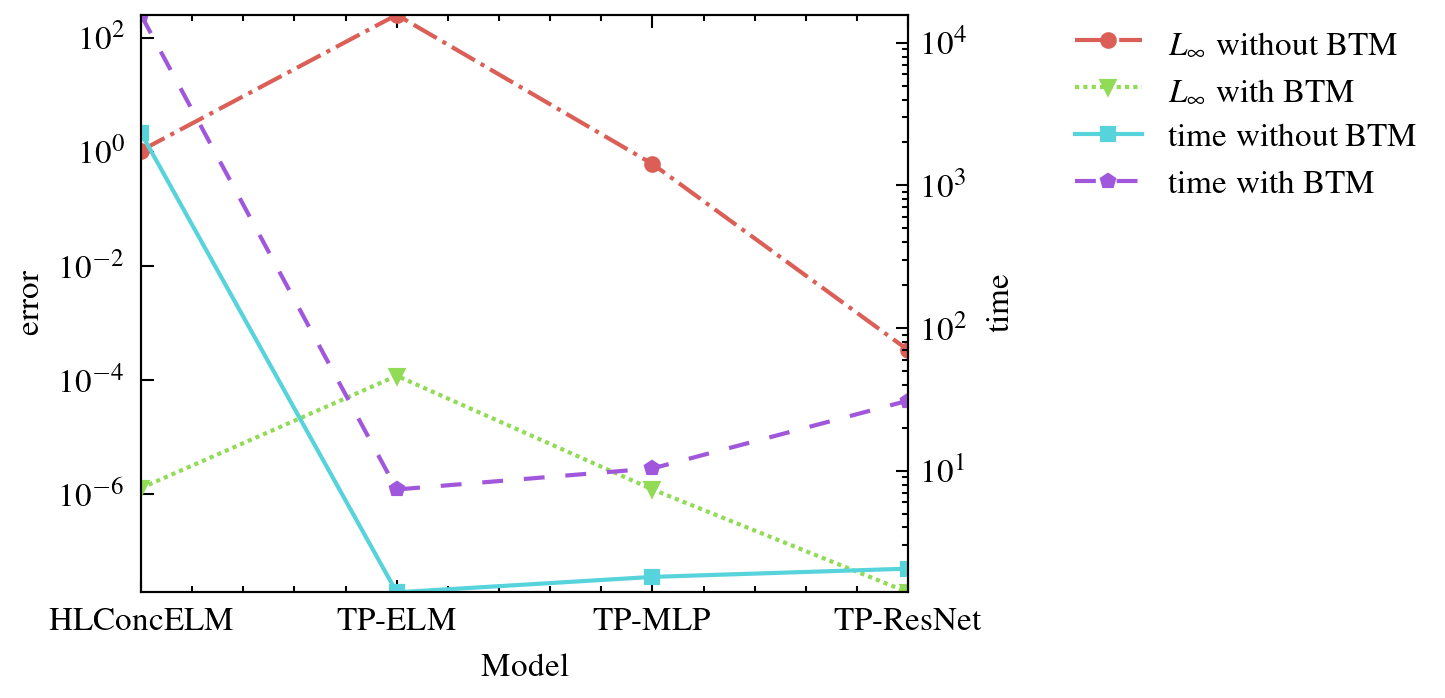}
    \end{minipage}
    \caption{Diffusion equation: The $L_{\infty}$ errors and training time for the TPNet and HLConcELM \cite{ni2023numerical} when solving the equation~\eqref{eq:Diffusion_equation}.}
    \label{fig:error_Diff_tf10}
\end{figure}

\begin{table}[htp]
    \begin{center}
        \caption{Diffusion equation: Approximation errors and training time for TPNet and HLConcELM \cite{ni2023numerical} when solving Equation~\eqref{eq:Diffusion_equation}.}
        \setlength\tabcolsep{1pt}
        \tiny{
        \begin{tabular}{ccccccc}
            \hline\noalign{\smallskip}
            \multirow{2}{*}{Model} & \multicolumn{3}{c}{Without BTM} & \multicolumn{3}{c}{With BTM}  \\
            & $L_{\infty}$ & $L_{2}$ & time (s) & $L_{\infty}$ & $L_{2}$ & time (s)    \\
            \hline
            HLConcELM \cite{ni2023numerical}   & 1.05E+00 & 3.37E+01 & 2333.6977   & 1.29E-06 & 1.15E-04 & 15682.4688  \\
            TP-ELM                             & 2.49E+02 & 9.21E+03 & 1.4039      & 1.19E-04 & 1.19E-02 & 7.3839  \\
            TP-MLP                             & 6.10E-01 & 2.02E+01 & 1.8010      & 1.22E-06 & 1.29E-04 & 10.3772  \\
            TP-ResNet                          & 3.35E-04 & 1.07E-02 & 2.0580      & 1.92E-08 & 1.61E-06 & 31.0119  \\
            \hline
        \end{tabular}
        }
        \label{tab:error_Diff_tf10}
    \end{center}
\end{table}

\begin{figure}[htp]
    \centering
    \includegraphics[width=1\textwidth]{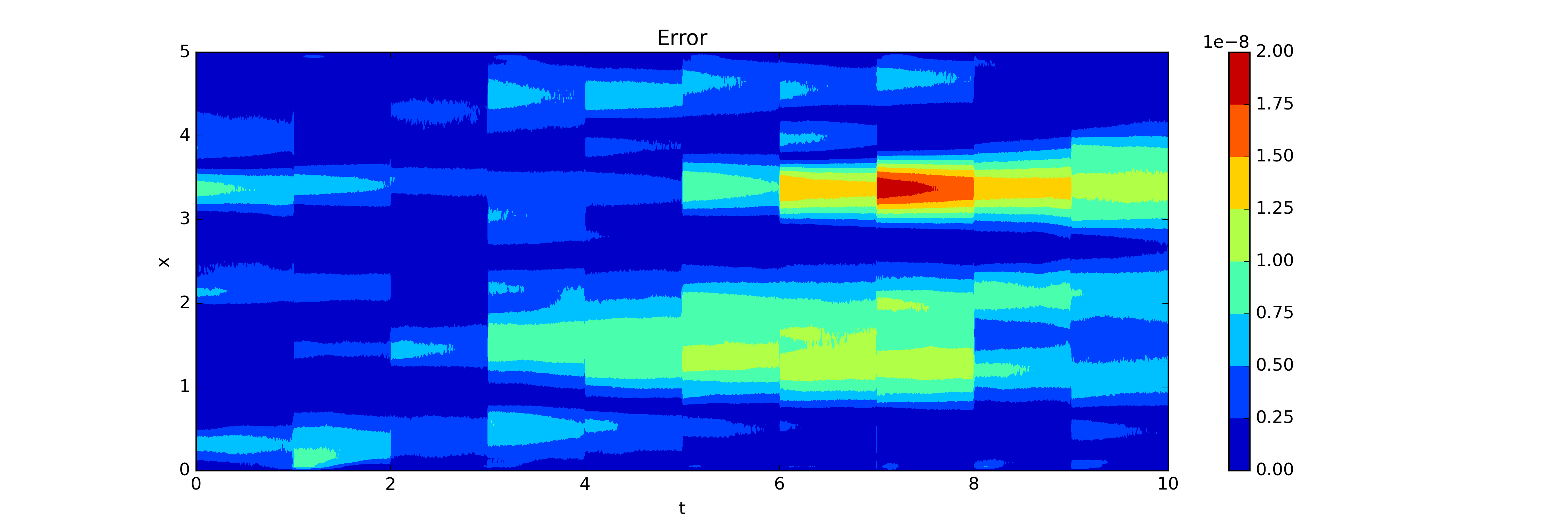}
    \caption{Diffusion equation: Absolute error between the exact solution and approximate solutions to the diffusion Equation~\eqref{eq:Diffusion_equation} with $M=10,000$ and $t_f=10$.}
    \label{fig:heatmap3_diffusion_10}
\end{figure}

% 总结 Conclusions
\section{Conclusions and Discussions}
\label{sec:summary}
In this work, we propose a novel neural network architecture called TPNet for approximating functions and solving both linear and nonlinear PDEs. TPNet constructs a set of basis functions by computing the tensor product of outputs from two subnetworks, each of which represents a distinct set of basis functions. The target function is then expressed as a linear combination of these basis functions with the coefficients determined determined by the least squares method. Unlike traditional backpropagation-based training, this approach significantly improves computational efficiency while maintaining high solution accuracy. 

To train the neural network, we employ a set of collocation points sampled within the problem domain. In our experiments, the two subnetworks of TPNet adopt identical architectures inspired by ELM, MLP, and ResNet. A detailed comparison is conducted between TPNet and HLConcELM. Although uniform grids of collocation points are primarily adopted for simplicity (except in the high-dimensional Poisson problem), the method remains equally effective with randomly sampled points. This adaptability enables TPNet to handle geometrically complex domains beyond regular regions. Furthermore, the required derivatives in the PDEs are computed using automatic differentiation. Notably, TPNet demonstrates higher efficiency in automatic differentiation compared to HLConcELM. This advantage stems from TPNet’s compact network architecture, whereas HLConcELM requires a substantially larger network. Due to the effect of computational cost with network size in automatic differentiation, TPNet achieves significant efficiency gains, as numerically validated in our experiments.

Furthermore, our training strategy integrates collocation points not only within the domain interior but also at boundary and initial condition locations. By substituting these points into the governing Equations \eqref{eq:linear_system_left} and \eqref{eq:linear_system_right} and determining the coefficients via the least squares method, TPNet ensures that physical constraints and initial conditions are satisfied. This systematic enforcement ensures high accuracy and robustness in the PDE solutions.

For long-time simulations of time-dependent PDEs, we propose a block time-marching strategy that divides the temporal domain into uniform sequential blocks. This decomposition transforms the original problem into a series of initial-boundary value subproblems, each assigned to a discrete time block. TPNet solves these subproblems temporally in sequence, improving computational tractability and ensuring stable neural network training. The synergy between TPNet’s computational efficiency and the block time-marching strategy enables accurate and scalable long-time simulations, as numerically demonstrated in our diffusion equation case study.

Despite the promising results achieved with TPNet, several aspects warrant further investigation to enhance its performance. First, while our numerical methods are effective, they do not yet achieve machine-level precision. Future work should explore strategies to further refine the algorithm’s accuracy. Second, our experiments indicate a correlation between accuracy and the number of basis functions, but accuracy improvements plateau beyond a certain threshold. Understanding this relationship is crucial for systematically enhancing the method’s precision. Finally, the neural network architecture plays a key role in solution accuracy, yet many architectures fail to yield better results. This suggests potential for improving TPNet’s representational capacity through optimized network designs. In summary, TPNet has demonstrated strong potential in solving PDEs efficiently and accurately. However, further research is required to address these challenges and fully optimize the method for broader applications.

% 致谢 Acknowledgement
% \clearpage
\section*{Acknowledgment}
This research is partially supported by the National Natural Science Foundation of China (No.U25A20200, No.12371434), and the National Key R \& D Program of China (No.2022YFE03040002).

\section*{Data Availability Statement}
The data that support the findings of this study are available from the corresponding author upon reasonable request.

\section*{Conflict of Interest}

The authors have no conflicts to disclose.

% \bibliographystyle{cas-model2-names}
% \bibliographystyle{elsarticle-num}

\bibliographystyle{unsrt}
% \bibliography{./refs}
\bibliography{./ref}
\end{document}